\documentclass{article}

\usepackage[final,nonatbib]{nips_2018}

\usepackage[utf8]{inputenc} 
\usepackage[T1]{fontenc}    
\usepackage{hyperref}       
\usepackage{url}            
\usepackage{booktabs}       
\usepackage{amsfonts}       
\usepackage{nicefrac}       
\usepackage{microtype}      
\usepackage{amsmath}
\usepackage{amsthm}
\usepackage{amssymb}
\usepackage{graphicx}
\usepackage{algorithm,float}

\usepackage{adjustbox}
\usepackage{bbm}
\usepackage{xcolor}
\usepackage{colortbl,enumitem}
\usepackage{appendix}

\floatstyle{ruled}
\newfloat{algorithm}{tb}{loa}
\providecommand{\algorithmname}{Policy}
\floatname{algorithm}{\protect\algorithmname}

\newtheorem{assumption}{Assumption}
\newtheorem{defn}{Definition}
\newtheorem{lem}{Lemma}
\newtheorem{thm}{Theorem}
\newtheorem{prop}{Proposition}
\newtheorem{cor}{Corollary}

\global\long\def\real{\mathbb{R}}
\global\long\def\P{\Pr}
\global\long\def\E{\mathbb{E}}
\global\long\def\indic{\mathbbm{1}}

\global\long\def\ball{\mathcal{B}}
\global\long\def\mA{\mathcal{A}}
\global\long\def\mX{\mathcal{X}}

\global\long\def\mF{\mathcal{F}}

\global\long\def\mZ{\mathcal{Z}}
\global\long\def\mY{\mathcal{Y}}
\global\long\def\norm#1#2{\rho\left(#1,#2\right)}

\global\long\def\Ot{\widetilde{O}}
\global\long\def\Omegat{\widetilde{\Omega}}

\definecolor{darkgreen}{rgb}{0.0, 0.8, 0.2}
\definecolor{lightgray}{rgb}{0.8, 0.85, 0.85}

\title{Q-learning with Nearest Neighbors}

\author{
	Devavrat Shah \thanks{Both authors are affiliated with Laboratory for Information and Decision Systems (LIDS). DS is 
		with the Department of EECS as well as Statistics and Data Science Center at MIT. }  \\
	Massachusetts Institute of Technology\\
	\texttt{devavrat@mit.edu} \\
	\And
	Qiaomin Xie \footnotemark[1] \\
	Massachusetts Institute of Technology\\
	\texttt{qxie@mit.edu} \\
}

\begin{document}

\maketitle

\begin{abstract}
We consider model-free reinforcement learning for infinite-horizon discounted Markov Decision Processes (MDPs) 
with a \emph{continuous} state space and unknown transition kernel, when only a single sample path under an arbitrary policy 
of the system is available.  We consider the \emph{Nearest Neighbor Q-Learning} (NNQL) algorithm to learn the optimal
Q function using nearest neighbor regression method. As the main contribution, we provide tight finite sample analysis
of the convergence rate. In particular, for MDPs with a $d$-dimensional state space and the discounted factor 
$\gamma \in (0,1)$, given an arbitrary sample path with ``covering time'' $ L $, we establish that the algorithm is 
guaranteed to output an $\varepsilon$-accurate estimate of the optimal Q-function using  
$\Ot\big(L/(\varepsilon^3(1-\gamma)^7)\big)$ samples. For instance, for a well-behaved MDP, the covering time 
of the sample path under the purely random policy scales as $ \Ot\big(1/\varepsilon^d\big),$ so the sample complexity 
scales as $\Ot\big(1/\varepsilon^{d+3}\big).$ Indeed, we establish a lower bound that argues that the 
dependence of $ \Omegat\big(1/\varepsilon^{d+2}\big)$ is necessary. 

\end{abstract}

\section{Introduction}

Markov Decision Processes (MDPs) are natural models for a wide variety of sequential decision-making problems. It is well-known that the optimal control 
problem in MDPs can be solved, in principle, by standard algorithms such as value and policy iterations. These algorithms, however, are often not directly 
applicable to many practical MDP problems for several reasons. First, they do not scale computationally 
as their complexity grows quickly with the size of the state space and especially for continuous state space.
Second, in problems with complicated dynamics, the transition kernel of the underlying MDP is often unknown, or 
an accurate model thereof is lacking. To circumvent these difficulties, many model-free Reinforcement Learning (RL) 
algorithms have been proposed, in which one estimates the relevant quantities of the MDPs (e.g., the value functions or the optimal policies) from observed data generated by simulating the MDP. 

A popular model-free Reinforcement Learning (RL) algorithm is the so called Q-learning~\cite{Qlearning}, which directly learns 
the optimal action-value function (or Q function) from the observations of the system trajectories. A major advantage of Q-learning 
is that it can be implemented in an online, incremental fashion, in the sense that Q-learning can be run as data is being sequentially 
collected from the system operated/simulated under some policy, and continuously refines its estimates as new observations become 
available. The behaviors of standard Q-learning in {\em finite} state-action problems have by now been reasonably understood; in particular, 
both asymptotic and finite-sample convergence guarantees have been established~\cite{Tsitsiklis1994Qlearning,Jaakkola1994Qlearning,Szepesy1997Qlearning,Even-Dar2004Qlearning}. 

In this paper, we consider the general setting with \emph{continuous} state spaces. For such problems, existing algorithms typically make 
use of a parametric function approximation method, such as a linear approximation~\cite{Melo2008QLearningLinear}, to learn a compact representation of the action-value function.  
In many of the recently popularized applications of Q-learning, much more expressive function approximation method such as deep neural
networks have been utilized. Such approaches have enjoyed recent empirical success in game playing and robotics problems~\cite{Silver2016AlphaGo,Mnih2015Atari,Duan2016Robotics}.
Parametric approaches typically require careful selection of approximation method and parametrization (e.g., the architecture of neural networks). Further, rigorous convergence guarantees of Q-learning with deep neural networks are relatively less understood. 
In comparison, non-parametric approaches are, by design, more flexible and versatile. However, in the context of model-free RL with 
continuous state spaces, the convergence behaviors and finite-sample analysis of non-parametric approaches are  less understood. 

\noindent{\bf Summary of results.}  In this work, we consider a natural combination of the Q-learning with 
Kernel-based nearest neighbor regression for continuous state-space MDP problems, denoted as 
Nearest-Neighbor based Q-Learning (NNQL). As the main result, we provide {\em finite sample analysis} 
of NNQL for a {\em single, arbitrary} sequence of data for any infinite-horizon discounted-reward MDPs with continuous 
 state space. In particular, we show that the algorithm outputs an $\varepsilon$-accurate (with respect to supremum 
 norm) estimate 
of the optimal Q-function with high probability using a number of observations that depends polynomially on 
$ \varepsilon$, the model parameters and the ``cover time'' of the sequence of the data or trajectory of the data utilized.
For example, if the data was sampled per a completely random policy, then our generic bound suggests that 
the number of samples would scale as $\Ot(1/\varepsilon^{d+3})$ where $d$ is the dimension of the state space. 
We establish effectively matching lower bound stating that for {\em any} policy to learn optimal $Q$ function 
within $\varepsilon$ approximation, the number of samples required must scale as $\Omegat(1/\varepsilon^{d+2})$. 
In that sense, our policy is {\em nearly} optimal. 

Our analysis consists of viewing our algorithm as a special case of a general \emph{biased} stochastic 
approximation procedure, for which we establish non-asymptotic convergence guarantees. Key to our analysis 
is a careful characterization of the bias effect induced by nearest-neighbor approximation of the population 
Bellman operator, as well as the statistical estimation error due to the variance of finite, dependent samples.
Specifically, the resulting Bellman nearest neighbor operator allows us to connect the update rule of 
NNQL to a class of stochastic approximation algorithms, which have \emph{biased} noisy updates. 
Note that traditional results from stochastic approximation rely on unbiased updates 
and asymptotic analysis~\cite{Robbins1951stochastic,Tsitsiklis1994Qlearning}. 
A key step in our analysis involves decomposing the update into two sub-updates, which bears some similarity to the 
technique used by~\cite{Jaakkola1994Qlearning}. Our results make improvement in characterizing the finite-sample 
convergence rates of the two sub-updates. 

In summary, the salient features of our work are
\begin{itemize}[itemsep=0mm,leftmargin=*]
	\item \textbf{Unknown system dynamics:} We assume that 
	the transition kernel and reward function of the MDP 
	is unknown. Consequently, we cannot exactly evaluate the expectation required in standard dynamic programming algorithms (e.g., value/policy iteration). Instead, we consider a sample-based approach which learns the optimal value functions/policies by directly observing data generated by the MDP.  
	\item \textbf{Single sample path:} We are given a single, sequential samples obtained from the MDP operated under an
	arbitrary policy. 
	This in particular means that the observations used for learning are \emph{dependent}. Existing work often studies the easier settings where samples can be generated at will; that is, one can sample any number of (independent) transitions from any given state, or reset the system to any initial state. For example, \emph{Parallel Sampling} in~\cite{Kearns1999PQL}.
	We do not assume such capabilities, but instead deal with the realistic, challenging 
	setting with a single path.
	\item \textbf{Online computation:} We assume that data arrives sequentially rather than all at once. Estimates are updated in an online fashion upon observing each new sample. Moreover, as in standard Q-learning, our approach does not store old data. In particular, our approach differs from other batch methods, which need to wait for all data to be received before starting computation, and require multiple passes over the data. Therefore, our approach is space efficient, and hence can handle the data-rich scenario with a large, increasing number of samples.
	\item \textbf{Non-asymptotic, near optimal guarantees:} We characterize the finite-sample convergence rate of our 
	algorithm; that is, how many samples are needed to achieve a given accuracy for estimating the optimal value function. Our analysis is
	\emph{nearly} tight in that we establish a lower bound that nearly matches our generic upper bound specialized to setting when data is generated 
	per random policy or more generally any policy with random exploration component to it.
	\end{itemize}
While there is a large and growing literature on Reinforcement Learning for MDPs, to the best of our knowledge, ours
is the first result on Q-learning that simultaneously has all of the above four features. We summarize comparison with relevant prior
works in Table \ref{table:related-work}. Detailed discussion can be found in Appendix \ref{sec:related}. 

\begin{table}[t]
	\caption{Summary of relevant work. See Appendix \ref{sec:related} for details. }
	\label{table:related-work}
	\centering
	\begin{adjustbox}{width=1\textwidth,center=\textwidth}
		\begin{tabular}{p{52mm}cccccc}
			\toprule
			Specific work & Method & Continuous  & Unknown & Single &  Online & Non-asymptotic\\
			&   & state space &  transition Kernel
			& sample path & update &  guarantees \\
			\midrule
			
			\cite{Chow1989Complexity}, \cite{Rust1997}, \cite{Saldi2017} & Finite-state approximation & Yes & No & No & Yes & Yes\\
			\arrayrulecolor{lightgray}\midrule
				
			\cite{Tsitsiklis1994Qlearning}, \cite{Jaakkola1994Qlearning},
			\cite{Szepesy1997Qlearning} & Q-learning & No & Yes & Yes & Yes & No\\
			\midrule
			
			\cite{Hasselt2010DoubleQL}, \cite{Azar2011SQL},
			\cite{Even-Dar2004Qlearning} & Q-learning & No & Yes & Yes & Yes & Yes\\
			\midrule
			
			\cite{Kearns1999PQL} & Q-learning & No & Yes & No & Yes & Yes\\
			\midrule

			\cite{szepesvari2004interpolation},\cite{melo2007qlearning} & Q-learning & Yes & Yes & Yes & Yes & No\\
			\midrule
						
			\cite{Ormoneit2002Discounted}, \cite{Ormoneit2002Average}  & Kernel-based approximation & Yes & Yes & No & No & No\\			
			\midrule
			
			\cite{Haskell2016EDP} & Value/Policy iteration & No & Yes & No & No & Yes\\
			\midrule
			
			\cite{Tsitsiklis1997Linear} & Parameterized TD-learning & No & Yes & Yes & Yes & No\\
			\midrule

			\cite{dalal2017TD} & Parameterized TD-learning & No & Yes & No & Yes & Yes\\
			\midrule

			\cite{bhandari2018TD} & Parameterized TD-learning & No & Yes & Yes & Yes & Yes\\
			\midrule
			
			\cite{Bhat2012ADP_kernel} & Non-parametric LP & No & Yes & No & No & Yes\\
			\midrule
			
			\cite{Munos2008finite} & Fitted value iteration &  Yes & Yes & No & No & Yes \\
			\midrule
			
			\cite{Antos2008single} & Fitted policy iteration & Yes & Yes & Yes & No & Yes \\
			\midrule
			
 			Our work & Q-learning & Yes & Yes & Yes & Yes & Yes\\
			\arrayrulecolor{black}\bottomrule& 
		\end{tabular}
	\end{adjustbox}
\vspace{-.5cm}
\end{table}

\vspace{-.1in}
\section{Setup}
\label{sec:prelim}
\vspace{-.1in}

In this section, we introduce necessary notations, definitions for the framework of Markov Decision Processes that will be 
used throughout the paper. We also precisely define the question of interest. 

\medskip
\noindent{\bf Notation.} For a metric space $E$ endowed with metric $\rho$, we denote by $C(E)$ the set of all bounded and
measurable functions on $E$. For each $f\in C(E),$ let $\left\Vert f\right\Vert _{\infty}:=\sup_{x\in E}$$\left|f(x)\right|$
be the supremum norm, which turns $C(E)$ into a Banach space $B$.
Let $\text{Lip}(E,M)$ denote the set of Lipschitz continuous functions
on $E$ with Lipschitz bound $M$, i.e., 
\[
\text{Lip}(E,M)=\left\{ f\in C(E)\;|\;\left|f(x)-f(y)\right|\leq M\rho(x,y),\;\forall x,y\in E\right\}.
\]
The indicator function is denoted by $\indic\{\cdot\}$.
For each integer $ k \ge 0 $, let $ [k] \triangleq \{ 1,2, \ldots, k\}. $ 

\medskip
\noindent{\bf Markov Decision Process.} We consider a general  setting where an agent
interacts with a stochastic environment. This interaction is modeled
as a discrete-time discounted Markov decision process (MDP). An MDP
is described by a five-tuple $(\mX,\mA,p,r,\gamma)$, where $\mX$
and $\mA$ are the state space and action space, respectively. We shall
utilize $t \in \mathbb{N}$ to denote time. Let $x_t \in \mX$ be state
at time $t$. At time $t$, the action chosen is denoted as $a_t \in \mA$. 
Then the state evolution is Markovian as per some transition probability kernel 
with density $p$ (with respect to the Lebesgue measure $ \lambda $ on $ \mX $). That is, 
\begin{align}\label{eq:transition}
\Pr (x_{t+1} \in B | x_t = x, a_t = a) & = \int_B p(y | x, a) \lambda({d} y)
\end{align}
for any measurable set $ B \in \mX $.
The one-stage reward  earned at time $t$ is a random variable $R_t$ with expectation
$ \E [ R_t | x_t = x, a_t =a ] = r(x,a)$, where $r:\mX\times\mA\rightarrow\real$ is the 
expected reward function. Finally, $\gamma\in(0,1)$ is the discount 
factor and the overall reward of interest is 
$
\sum_{t=0}^\infty \gamma^t R_t
$
The goal is to maximize the expected value of this reward. 
Here we consider a distance function $\rho:\mX\times\mX\rightarrow\real_{+}$
so that $\left(\mX,\rho\right)$ forms a metric space. For the ease of exposition,
we use $\mZ$ for the joint state-action space $\mX\times\mA.$ 

We start with the following standard assumptions on the MDP:
\begin{assumption}[MDP Regularity]
\label{assu:MDP-Reularity} We assume that: (A1.) The continuous state space $\mX$ is a compact subset of $\real^{d}$; (A2.) $\mA$ is a finite set of cardinality $\left|\mA\right|$; (A3.) The one-stage reward $R_t$ is non-negative and uniformly bounded by $R_{\max}$, i.e., $0 \leq R_t \leq R_{\max}$ almost surely. For each $a\in\mA,$
	$r(\cdot,a)\in\textup{Lip}(\mX,M_{r})$ for some $M_{r}>0$. 
(A4.) The transition probability kernel $p$ satisfies 
	\[
	\left|p(y|x,a)-p(y|x',a)\right|\leq W_{p}(y)\norm x{x'},\qquad\forall a\in\mA,\forall x,x',y\in \mX,
	\]
	where the function $W_{p}(\cdot)$ satisfies 
	$
	\int_{\mX}W_{p}(y)\lambda(dy)\leq M_{p}.
	$

\end{assumption}

The first two assumptions state that the state space is compact and the action space is finite. The third and forth stipulate that the reward and transition kernel are Lipschitz continuous (as a function of the current state). Our Lipschitz assumptions are identical to (or less restricted than) those used in the work of~\cite{Rust1997}, \cite{Chow1991TACmultigrid}, and~\cite{Dufour2015StochasticsMDP}. In general, this type of Lipschitz continuity assumptions are standard in the literature on MDPs with continuous state spaces; see, e.g., the work of ~\cite{Dufour2012AnalAppl,Dufour2013LP}, and~\cite{Bertsekas1975TACdiscretization}. 

A Markov policy $\pi(\cdot|x)$ gives the probability of performing
action $a\in\mA$ given the current state $x$. A deterministic policy
assigns each state a unique action. The \emph{value} function for
each state $x$ under policy $\pi$, denoted by $V^{\pi}(x),$ is
defined as the expected discounted sum of rewards received following
the policy $\pi$ from initial state $x,$ i.e., $V^{\pi}(x)=\E_{\pi}\left[\sum_{t=0}^{\infty} \gamma^{t} R_t | x_0 = x \right]$.
The \emph{action-value function} $ Q^\pi $ under policy $\pi$ is defined
by $Q^{\pi}(x,a)=r(x,a)+\gamma\int_{y}p(y|x,a)V^{\pi}(y)\lambda(dy)$.
The number $Q^{\pi}(x,a)$ is called the \emph{Q-value} of the pair $(x,a)$, which is the return of initially performing action $a$ at state
$s$ and then following policy $\pi.$   Define 
$$\beta \triangleq 1/(1-\gamma) 
\qquad\text{and}
\qquad V_{\max} \triangleq \beta R_{\max}.		
$$ 
Since all the rewards are bounded by $R_{\max}$, it is easy to see that the value function of every policy is bounded by $V_{\max}$~\cite{Even-Dar2004Qlearning,Stehl2006PAC}.
The goal is to find an optimal policy $\pi^{*}$ that maximizes the
value from any start state. The optimal value function $V^{*}$is
defined as $V^{*}(x)=V^{\pi^{*}}(x)=\sup_{\pi}V^{\pi}(x)$, $\forall x\in\mX.$
The optimal action-value function is defined as $Q^{*}(x,a)=Q^{\pi^{*}}(x,a)=\sup_{\pi}Q^{\pi}(x,a)$.
The Bellman optimality operator $F$ is defined as
\begin{align*}
(FQ)(x,a) & =r(x,a)+\gamma\E\left[\max_{b\in\mA}Q(x',b)\;|\:x,a\right]=r(x,a)+\gamma\int_{\mX}p(y|x,a)\max_{b\in\mA}Q(y,b)\lambda(dy).
\end{align*}
It is well known that $F$ is a contraction with factor $\gamma$
on the Banach space $C(\mZ)$~\cite[Chap.\ 1]{Bertsekas1995DPOC}. 
The optimal action-value function $Q^{*}$ is the unique solution
of the Bellman's equation $Q=FQ$ in $C(\mX\times \mA).$  
In fact, under our setting, it can be show that $ Q^* $ is bounded and Lipschitz. This is stated below and established in
Appendix ~\ref{sec:proof_Q_Lip}.
\begin{lem}
	\label{lem:Q_Lip}
	Under Assumption~\ref{assu:MDP-Reularity}, the function $Q^{*}$ satisfies that
	$\left\Vert Q^{*}\right\Vert _{\infty}\le V_{\max}$ and that 
	 $Q^{*}(\cdot,a)\in\textup{Lip}(\mX,M_{r}+\gamma V_{\max}M_{p})$ for each $a\in\mA$.
\end{lem}

\section{Reinforcement Learning Using Nearest Neighbors}
\label{sec:algorithm}

In this section, we present the nearest-neighbor-based reinforcement learning algorithm. 
The algorithm is based on constructing a finite-state discretization of the original MDP, 
and combining Q-learning with nearest neighbor regression to estimate the $ Q $-values over 
the discretized state space, which is then interpolated and extended to the original continuous
state space.  In what follows, we shall first describe several building blocks for the algorithm in
Sections~\ref{sec:discretization}--\ref{sec:cover_time}, and then summarize the algorithm in Section~\ref{sec:NNQL}.   

\vspace{-.1in}
\subsection{State Space Discretization}
\label{sec:discretization}
\vspace{-.05in}

Let $h>0$ be a pre-specified scalar parameter. Since the state space $\mX$ is compact, one can find a finite set
$\mX_h \triangleq \{c_{i}\}_{i=1}^{N_{h}}$ of points in $\mX$ such that 
\[
\min_{i\in[N_{h}]}\rho(x,c_{i})<h,\;\forall x\in\mX.
\]
The finite grid $\mX_h $ is called an $h$-net
of $\mX,$ and its cardinality $n \equiv N_{h}$ can be chosen to be the $h$-covering number
of the metric space $(\mX,\rho).$ Define  $\mZ_{h} = \mX_{h}\times\mA$.
Throughout this paper, we denote
by $\ball_{i}$ the ball centered at $c_{i}$ with radius~$h$; that is, 
$\ball_{i} \triangleq \left\{ x\in\mX:\norm{x}{c_{i}}\leq h\right\}.$

\vspace{-.1in}
\subsection{Nearest Neighbor Regression}
\label{sec:NNR}
\vspace{-.05in}

Suppose that we are given estimated Q-values for the
finite subset of states $\mX_{h}=\{c_{i}\}_{i=1}^{n}$, denoted by
$q=\left\{ q(c_{i},a),c_{i}\in\mX_{h},a\in\mA\right\} $. For each
state-action pair $(x,a)\in\mX\times\mA,$ we can predict its Q-value
via a regression method. We focus on nonparametric regression operators
that can be written as nearest neighbors averaging in terms of the
data $q$ of the form 
\begin{equation}
(\Gamma_{\text{NN}}q)(x,a)= \textstyle{\sum_{i=1}^{n}} K(x,c_{i})q(c_{i},a),\qquad\forall x\in\mX,a\in\mA, \label{eq:NN_average}
\end{equation}
where $K(x,c_{i})\geq0$ is a weighting kernel function satisfying
$
\sum_{i=1}^{n}K(x,c_{i})=1, \forall x\in \mX.
$
Equation~\eqref{eq:NN_average} defines the so-called Nearest Neighbor (NN) operator $ \Gamma_{\text{NN}} $, which maps the space $ C(\mX_{h}\times\mA) $ into the set of all bounded function over $ \mX \times \mA $. 
Intuitively, in (\ref{eq:NN_average}) one assesses the Q-value of $(x,a)$ by looking at the training data where 
the action $a$ has been applied, and by averaging their values. It can be easily checked that the operator $\Gamma_{\text{NN}}$ is non-expansive in the following sense:
\begin{align}
\label{eq:nn-contract}
\left\Vert \Gamma_{\text{NN}}q-\Gamma_{\text{NN}}q'\right\Vert _{\infty} & \leq\left\Vert q-q'\right\Vert _{\infty},\qquad\forall q,q'\in C(\mX_{h}\times\mA).
\end{align}
This property will be crucially used for establishing our results. $ K $ is assumed to satisfy 
\begin{equation}
\label{eq:weight_assumption}
K(x,y)=0~~\text{if}~~\rho(x,y)\geq h, \qquad \forall x\in \mX, y \in \mX_{h},
\end{equation}
where $ h $ is the discretization parameter defined in Section~\ref{sec:discretization}.\footnote{This assumption is not absolutely necessary, but is imposed to simplify subsequent analysis. In general, our results hold as long as $ K(x,y) $ decays sufficiently fast with the distance $ \rho(x,y) $. }
This means that the values of states located in the neighborhood of $x$ are more influential in the averaging procedure~\eqref{eq:NN_average}. There are many possible choices for $ K $. In Section~\ref{sec:NN_examples} we describe three representative choices that correspond to $k$-Nearest Neighbor Regression, 
Fixed-Radius Near Neighbor Regression and Kernel Regression.

\vspace{-.1in}
\subsection{A Joint Bellman-NN Operator}
\label{sec:joint_operator}
\vspace{-.05in}

Now, we define the \emph{joint Bellman-NN (Nearest Neighbor) operator}. As will become clear subsequently, it is this operator that 
the algorithm aims to approximate, and hence it plays a crucial role in the subsequent analysis.  

For a function $q:\mZ_{h}\rightarrow\real,$ we denote by $\tilde{Q} \triangleq (\Gamma_{\text{NN}}q)$
the nearest-neighbor average extension of $q$ to $\mZ$; that is,
\[
\tilde{Q}(x,a)=(\Gamma_{\text{NN}}q)(x,a), \quad \forall(x,a)\in\mZ.
\]
The joint Bellman-NN operator $G$ on $ \real^{|\mZ_h|} $ is defined by composing
the original Bellman operator $F$ with the NN operator $ \Gamma_{\text{NN}} $ and then restricting to $\mZ_{h}$; that is, for each $(c_{i},a)\in\mZ_{h}$,
\begin{align}
(Gq)(c_{i},a) \triangleq (F\Gamma_{\text{NN}}q)(c_{i},a) &= (F\tilde{Q})(c_{i},a) 
 =r(c_{i},a)+\gamma\E\left[\max_{b\in\mA}(\Gamma_{\text{NN}}q)(x',b)\;|\;c_{i},a\right]. \label{eq:Bellman-NN-operator} 
\end{align}
It can be shown that $G$ is a contraction operator with modulus $\gamma$
mapping $ \real^{|\mZ_h|} $ to itself, thus admitting a unique fixed point,
denoted by $q_{h}^{*}$; see Appendix~\ref{subsec:properties_NNQL}.

\vspace{-.1in}
\subsection{Covering Time of Discretized MDP}
\label{sec:cover_time}
\vspace{-.05in}

As detailed in Section~\ref{sec:NNQL} to follow, our algorithm uses data generated by an abritrary policy $ \pi $ for the purpose of learning.
The goal of our approach is to estimate the Q-values of \emph{every} state. For there
to be any hope to learn something about the value of a given state,
this state (or its neighbors) must be visited at least once. Therefore, 
to study the convergence rate of the algorithm, we need a way to quantify how often $ \pi $ samples from different regions of the state-action space $ \mZ = \mX \times \mA $.

Following the approach taken by~\cite{Even-Dar2004Qlearning} and~\cite{Azar2011SQL}, 
we introduce the notion of the \emph{covering time} of MDP under a policy $ \pi $.
This notion is particularly suitable for our setting as our algorithm is based on \emph{asynchronous} Q-learning (that is, we are given a single, sequential trajectory of the MDP, where at each time step one state-action pair is observed and updated), and the policy $ \pi $ may be non-stationary.  
In our continuous state space setting, the covering time is defined with respect to the discretized space $\mZ_{h}$, as follows:
\begin{defn}[Covering time of discretized MDP]
	For each $ 1\le i \le n=  N_h $ and $ a\in\mA $, a ball-action pair $(\ball_{i},a)$ is said to be visited at
	time $t$ if $x_t \in \ball_i$ and $a_t = a$.
	The discretized state-action space $ \mZ_{h} $ is covered by the policy $\pi$
	if all the ball-action pairs are visited at least once under the policy $\pi$. 
	Define $\tau_{\pi,h}(x,t)$, the covering time of the MDP under the 
	policy $\pi$, as the minimum number of steps required to visit all
	ball-action pairs starting from state $x\in\mX$ at time-step
	$t\geq0.$ Formally, $ \tau_{\pi,h}(x,t)  $ is defined as
	\begin{align*}
		\min\Big\{s \geq 0: x_t \!=\! x, ~\forall i \!\leq\! N_h, a \!\in \! \mA, ~\exists  t_{i, a} \!\in\! [t, t\!+\!s], \mbox{~such~that~} x_{t_{i, a}} \!\in\! B_i 
		\mbox{ and }  a_{t_{i, a}} \!=\! a, \mbox{~under~} \pi\Big\}, 
	\end{align*}
with notation that minimum over empty set is $\infty$. 
\end{defn}	
	
We shall assume that there exists a policy $\pi$ with bounded expected cover
time, which guarantees that, asymptotically, all the ball-action pairs are visited infinitely 
many times under the policy $\pi.$ 
\begin{assumption}
	\label{assu:cover-time}There exists an integer $L_{h} <\infty$
	such that
	$
	\E[\tau_{\pi,h}(x,t)]\leq L_{h},
	$
	$\forall x\in\mX$, $t>0.$ Here the expectation is defined with
	respect to randomness introduced by Markov kernel of MDP as well as the policy $\pi$. 
\end{assumption}

In general, the covering time can be large in the worst case. 
In fact, even with a finite state space, it is easy to find examples where the covering time is exponential
in the number of states for every
policy. For instance, consider
an MDP with states $1, 2,\ldots, N$, where at any state~$i$, the chain is
reset to state~1 with probability $1/2$ regardless of the action taken.
Then, every policy takes exponential time to reach state $N$
starting from state $ 1 $, leading to an exponential covering time.

To avoid the such bad cases, some additional assumptions are needed to 
ensure that the MDP is well-behaved. For such MDPs, there are a variety of 
polices that have a small covering time. Below we focus on a class of MDPs 
satisfying a form of the uniform ergodic assumptions, and show that the 
standard $ \varepsilon $-greedy policy (which includes the purely random 
policy as special case by setting $ \varepsilon =1 $) has a small covering time. 
This is done in the following two Propositions. Proofs can be found in 
Appendix \ref{sec:proof_cover_time}.

\begin{prop} \label{prop:cover_time}
	Suppose that the MDP satisfies the following: there exists a probability measure $ \nu $ on $ \mX $, a number $ \varphi >0 $ and an integer $ m \ge 1$ such that for all $ x \in \mX ,$ all $t\geq 0$ and all policies $ \mu $,
	\begin{align}\label{eq:unif.erg}
		\textstyle{\Pr_\mu} \left( x_{m+t} \in \cdot \right \vert x_t = x) & \ge \varphi \nu (\cdot).
	\end{align} 
	Let $ \nu_{\min} \triangleq \min_{i\in[n]} \nu(\ball_i) $, where we recall that $ n \equiv N_{h} = |\mX_{h} |$ is the cardinality of the discretized state space.
	Then the expected covering time of $ \varepsilon $-greedy is upper bounded by $L_h = O\Big( \frac{m|\mA|}{\varepsilon \varphi \nu_{\min}} \log (n|\mA|)\Big)$.
\end{prop}

\begin{prop} \label{prop:cover_time_single}
	Suppose that the MDP satisfies the following: there exists a probability measure $ \nu $ on $ \mX $, a number $ \varphi >0 $ and an integer $ m \ge 1$ such that for all $ x \in \mX ,$ all $t\geq 0$, there exists a sequence of actions $\boldsymbol{\hat{a}}(x)=(\hat{a}_1,\ldots,\hat{a}_m)\in\mA^m$,
	\begin{align}\label{eq:unif.erg.single}
		\textstyle{\Pr} \left( x_{m+t} \in \cdot \right \vert x_t = x, a_t=\hat{a}_1,\ldots,a_{t+m-1}=\hat{a}_m) & \ge \varphi \nu (\cdot).
	\end{align} 
	Let $ \nu_{\min} \triangleq \min_{i\in[n]} \nu(\ball_i) $, where we recall that $ n \equiv N_{h} = |\mX_{h} |$ is the cardinality of the discretized state space.
	Then the expected covering time of $ \varepsilon $-greedy is upper bounded by $L_h = O\Big( \frac{m|\mA|^{m+1}}{\varepsilon^{m+1} \varphi \nu_{\min}} \log (n|\mA|)\Big)$.
\end{prop}

\vspace{-.1in}
\subsection{Q-learning using Nearest Neighbor}
\label{sec:NNQL}
\vspace{-.05in}

We describe the nearest-neighbor Q-learning (NNQL) policy. 
Like Q-learning, it is a model-free policy for solving
MDP. Unlike standard Q-learning, it is (relatively) efficient to implement as 
it does not require learning the Q function over entire space $\mX \times \mA$. 
Instead, we utilize the nearest neighbor regressed Q function 
using the learned Q values restricted to $\mZ_h$. The policy assumes
access to an existing policy $\pi$ (which is sometimes called the ``exploration policy'', and need not have any optimality properties) 
that is used to sample data points for learning. 

The pseudo-code of NNQL is described in Policy \ref{alg:-Q-learning}.
At each time step $t$, action $a_{t}$ is performed from state $Y_{t}$ as per
the given (potentially non-optimal) policy $\pi$, and 
the next state $Y_{t+1}$ is generated according to $p(\cdot|Y_{t},a_{t}).$ 
Note that the sequence of observed states $(Y_{t})$ take continuous values in the 
state space $\mX$. 

\begin{algorithm}[h]
	\caption{\label{alg:-Q-learning}Nearest-Neighbor Q-learning}
	
	\textbf{Input}: Exploration policy $\pi$, discount factor
	$\gamma$, number of steps $T$, bandwidth parameter $ h $, and initial state~$Y_{0}$.
	
	Construct discretized state space $ \mX_{h} $;  initialize 
	$t = k = 0, ~\alpha_{0}=1,~q^0 \equiv 0;$

	\textbf{Foreach $(c_{i},a)\in\mZ_{h}$,  set} $N_{0}(c_{i},a)=0;$ \textbf{end}
	
	\textbf{repeat}
	
	$\qquad$Draw action $a_{t}\sim\pi(\cdot|Y_{t})$ and observe reward $ R_t $; generate the next state $Y_{t+1} \sim p(\cdot|Y_{t},a_{t});$
	
	$\qquad$\textbf{Foreach $i$ such that $Y_{t}\in\ball_{i}$ do}
	
	$\qquad\qquad$$\eta_{N}=\frac{1}{N_{k}(c_{i},a_{t})+1};$
	
	$\qquad\qquad$\textbf{if }$N_{k}(c_{i},a_{t})>0$ \textbf{then }
	
	$\qquad$$\qquad\quad$$(G^{k}q^{k})(c_{i},a_{t})=(1-\eta_{N})(G^{k}q^{k})(c_{i},a_{t})+\eta_{N}\left(R_t+\gamma\max_{b\in\mA}(\Gamma_{\text{NN}}q^{k})(Y_{t+1},b)\right)$;
	
	$\qquad\qquad$\textbf{else} $(G^{k}q^{k})(c_{i},a_{t}) = R_t +\gamma\max_{b\in\mA}(\Gamma_{\text{NN}}q^{k})(Y_{t+1},b)$;
	
	$\qquad\qquad$\textbf{end}

	$\qquad\qquad$ $N_{k}(c_{i},a_{t}) = N_{k}(c_{i},a_{t}) + 1$

	\textbf{$\qquad$end}
	
	$\qquad$\textbf{if $\min_{(c_{i},a)\in\mZ_{h}}N_{k}(c_{i},a)>0$
		then }
	
	$\qquad\qquad$\textbf{Foreach $(c_{i},a)\in\mZ_{h}$ do}
	
	$\qquad\qquad$$\qquad$$q^{k+1}(c_{i},a) =(1-\alpha_{k})q^{k}(c_{i},a)+\alpha_{k}(G^{k}q^{k})(c_{i},a)$;
	
	$\qquad\qquad$\textbf{end}
	
	$\qquad\qquad$$k =k+1; \alpha_{k}=\frac{\beta}{\beta+k}$;

	$\qquad\qquad$\textbf{Foreach $(c_{i},a)\in\mZ_{h}$ do $N_{k}(c_{i},a)=0;$}
	\textbf{end}
	
	$\qquad$\textbf{end}
	
	$\qquad t =t+1;$
	
	\textbf{until} $t\geq T$;
	
	return $\hat{q} = q^{k}$
	
\end{algorithm}

The policy runs over \emph{iteration} with each iteration lasting for a number
of time steps. Let $k$ denote iteration count, $T_k$ denote time when iteration $k$
starts for $k \in \mathbb{N}$. Initially, $k = 0$, $T_0 = 0$, and for $t \in [T_k, T_{k+1})$, 
the policy is in iteration~$k$. The iteration is updated from $k$ to $k+1$ 
when starting with $t = T_k$, all ball-action $(\ball_i,a)$ pairs have been visited at least 
once.  That is, 
$
T_{k+1}  = T_k + \tau_{\pi, h}(Y_{T_k}, T_k).
$
In the policy description, the counter $N_k(c_i, a)$ records how many times the ball-action pair $ (\ball_i, a) $ has been visited from the beginning of iteration $ k $ till the current time $ t $; that is, $N_k(c_i, a) = \sum_{s=T_k}^{t} \indic \{Y_s \in \ball_{i}, a_s = a\}.$
By definition, the iteration $k$ ends at the first time step for which $\min_{(c_i, a)} N_k(c_i, a) > 0$.

During each iteration, the policy keeps track of the Q-function over the finite set $\mZ_{h}$. Specifically, let $q^k$ denote the approximate Q-values on $\mZ_{h}$ within iteration $k$.  The policy also maintains $ G^{k}q^k(c_i,a_t) $, which is a \emph{biased} empirical estimate of the joint
Bellman-NN operator $G$ applied to the estimates $ q^k $. At each time step $t \in [T_k, T_{k+1})$ within iteration~$k$, if the current state $ Y_t $ falls in the ball $ \ball_i $, then the corresponding value $(G^{k}q^{k})(c_{i},a_{t})$ is updated as
\begin{align}
(G^{k}q^{k})(c_{i},a_{t}) & =(1-\eta_{N})(G^{k}q^{k})(c_{i},a_{t})+\eta_{N} \Big( R_t +\gamma\max_{b\in\mA}(\Gamma_{\text{NN}}q^{k})(Y_{t+1},b)\Big), \label{eq:Gk}
\end{align}
where $\eta_N = \frac{1}{N_{k}(c_{i},a_{t})+1}$. We notice that the above update rule computes, in an incremental fashion, an estimate of the joint Bellman-NN operator $ G  $ applied to the current $ q^k $ for each discretized state-action pair $ (c_i, a) $, using observations $ Y_t $ that fall into the neighborhood $ \ball_i $ of $ c_i $. This nearest-neighbor approximation causes the estimate to be biased. 

At the end of iteration $k$, i.e., at time step $t = T_{k+1}-1$, a new $q^{k+1}$ is generated as follows: 
for each $(c_i, a) \in \mZ_h$,
\begin{align}
q^{k+1}(c_{i},a) & =(1-\alpha_{k})q^{k}(c_{i},a)+\alpha_{k}(G^{k}q^{k})(c_{i},a). \label{eq:NNQL}
\end{align}
At a high level, this update is similar to standard Q-learning updates --- the Q-values are updated by taking a weighted average of $ q^k $, the previous estimate, and $ G^k q^k $, an one-step application of the Bellman operator estimated using newly observed data. There are two main differences from standard Q-learning: 1) the Q-value of each $ (c_i, a) $ is estimated using all observations that lie \emph{in its neighborhood} --- a key ingredient of our approach; 2) we wait until all ball-action pairs are visited to update their Q-values, all at once.

Given the output $\hat{q}$ of Policy \ref{alg:-Q-learning}, we obtain an approximate Q-value for each (continuous) state-action pair $(x,a)\in\mZ$
via the nearest-neighbor average operation, i.e., $Q_{h}^{T}(x,a)=\left(\Gamma_{\text{NN}}\hat{q}\right)(x,a);$
here the superscript $ T $ emphasizes that the algorithm is run for $ T $ time steps with a sample size of~$ T $.

\section{Main Results}
\label{sec:main}

As a main result of this paper, we obtain finite-sample analysis of NNQL policy. Specifically, we find that the NNQL policy converges to 
an $\varepsilon$-accurate estimate of the optimal $Q^*$ with time $T$ that has polynomial dependence on the model parameters. The proof 
can be found in Appendix~\ref{sec:proof_main}.
\begin{thm}
	\label{thm:NN Q-learning} Suppose that Assumptions \ref{assu:MDP-Reularity} and \ref{assu:cover-time} hold. With notation $\beta = 1/(1-\gamma)$ and $ C = M_{r}+\gamma V_{\max}M_{p} $, 
	for a given $\varepsilon \in (0,4V_{\max}\beta)$, 
	define 
	$h^* \equiv h^{*}(\varepsilon)= \frac{\varepsilon}{4\beta C}.$
 Let $N_{h^{*}}$ be the $h^*$-covering number of the
	metric space $(\mX,\rho).$ For a universal constant $C_0>0$, after at most
	\[
	T=C_0 \frac{L_{h^{*}}V_{\max}^{3}\beta^{4}}{\varepsilon^{3}}\log\left(\frac{2}{\delta}\right)\log\left(\frac{N_{h^{*}}\left|\mA\right|V^2_{\max}\beta^{4}}{\delta\varepsilon^{2}}\right)
	\]
	steps, with probability at least $1-\delta$, we have $\left\Vert Q_{h^{*}}^{T}-Q^{*}\right\Vert _{\infty}\leq\varepsilon.$
\end{thm}

The theorem provides sufficient conditions for NNQL to achieve $ \varepsilon $ accuracy (in sup norm) for estimating the optimal action-value function $ Q^* $.
The conditions involve the bandwidth parameter $ h^* $ and the number of time steps $ T $, both of which depend polynomially on the relevant problem parameters.  
Here an important parameter is the covering number $ N_{h^*} $: it provides a measure of the ``complexity'' of the state space $ \mX $, replacing the role of the cardinality $ |\mX| $ in the context of discrete state spaces. For instance, for a unit volume ball in $ \real^d, $ the corresponding covering number $ N_{h^*} $ scales as $ O\big((1/h^*)^d\big)$~(cf. Proposition~4.2.12 in~\cite{vershynin_hdpbook}).
We take note of several remarks on the implications of the theorem.

\textbf{Sample complexity:} The number of time steps $ T $, which also equals the number of samples needed, scales linearly with the covering time $ L_{h^*} $ of the underlying policy $\pi$ to sample data for the given MDP. Note that $ L_{h^*} $ depends implicitly on the complexities of the state and action space as measured by $N_{h^*} $ and $ |\mA| $. In the best scenario, $ L_{h^*} $, and hence $T $ as well, is linear in $ N_{h^*} \times |\mA| $ (up to logarithmic factors), in which case we achieve (near) optimal linear sample complexity. The sample complexity $ T $ also depends polynomially on the desired accuracy $ \varepsilon^{-1} $ and the effective horizon $ \beta = 1/(1-\gamma) $ of the discounted MDP --- optimizing the exponents of the polynomial dependence remains interesting future work.

\textbf{Space complexity:} The space complexity of NNQL is $ O(N_{h^*} \times |\mA| ) $, which is necessary for storing the values of $ q^k $. Note that NNQL is a truly online algorithm, as each data point $ (Y_t, a_t) $ is accessed only once upon observation and then discarded; no storage of them is needed.

\textbf{Computational complexity:} In terms of computational complexity, the algorithm needs to compute the NN operator $ \Gamma_{\text{NN}} $ and maximization over $ \mA $ in each time step, as well as to update the values of $ q^k $ for all $ c_i \in \mX_{h^*} $ and $a \in \mA $ in each iteration. Therefore, the worst-case computational complexity per time step is $ O(N_{h^*} \times |\mA| ) $, with an overall complexity of $ O(T \times N_{h^*} \times |\mA| ) $. The computation can be potentially sped up by using more efficient data structures and algorithms for finding (approximate) nearest neighbors, such as k-d trees~\cite{Bentley1979kdtree}, random projection trees~\cite{Dasgupta2008random}, Locality Sensitive Hashing~\cite{Indyk1998LSH} and boundary trees~\cite{Mathy2015boundary}.  

\textbf{Choice of $ h^* $:} NNQL requires as input a user-specified parameter $ h $, which determines the discretization granularity of the state space as well as the bandwidth of the (kernel) nearest neighbor regression. Theorem~\ref{thm:NN Q-learning} provides a desired value $ h^* = \varepsilon/4\beta C$, where we recall that $ C $ is the Lipschitz parameter of the optimal action-value function $ Q^* $ (see Lemma~\ref{lem:Q_Lip}). Therefore, we need to use a small $ h^* $ if we demand a small error $ \varepsilon $, or if $ Q^* $ fluctuates a lot with a large $ C $.

\subsection{Special Cases and Lower Bounds}

Theorem~\ref{thm:NN Q-learning}, combined with Proposition~\ref{prop:cover_time}, immediately yield the following bound that quantify the number of samples required to obtain an $\varepsilon$-optimal action-value function with high probability, if the sample path is generated per the uniformly random policy. The proof is given in Appendix~\ref{sec:proof_random_policy}.
\begin{cor}\label{cor:random_policy}
	Suppose that Assumptions \ref{assu:MDP-Reularity} and \ref{assu:cover-time} hold, with $\mX=[0,1]^d.$ Assume that the MDP satisfies the following: there exists a uniform probability measure $ \nu $ over $ \mX $, a number $ \varphi >0 $ and an integer $ m \ge 1$ such that for all $ x \in \mX ,$ all $t\geq 0$ and all policies $ \mu $,
	$
	\textstyle{\Pr_\mu} \left( x_{m+t} \in \cdot \right \vert x_t = x)  \ge \varphi \nu (\cdot).
	$
	After at most 
	\[
	T= \kappa \frac{1}{\varepsilon^{d+3}} \log^3\left(\frac{1}{\delta\varepsilon}\right)
	\]
	steps, where $ \kappa \equiv \kappa(|\mA|, d, \beta,m) $ is a number independent of $ \varepsilon $ and $ \delta $,  we have $\left\Vert Q_{h^{*}}^{T}-Q^{*}\right\Vert _{\infty}\leq\varepsilon$ with probability at least $1-\delta$.
\end{cor}

Corollary~\ref{cor:random_policy} states that 
the sample complexity of NNQL scales as $ \Ot\big(\frac{1}{\varepsilon^{d+3}}\big).$ We will show that this is effectively necessary by establishing a
lower bound on {\em any} algorithm under {\em any} sampling policy! The proof of Theorem~\ref{thm:lower_bound} can be found in Appendix~\ref{sec:proof_lower_bound}.
\begin{thm} \label{thm:lower_bound}
	For any reinforcement learning algorithm $\hat{Q}_T$ and any number
	$\delta\in(0,1)$, there exists an MDP problem and some number $T_{\delta}>0$
	such that 
	\[
	\P\bigg[ \big\Vert \hat{Q}_T-Q^{*}\big\Vert_{\infty}\ge C\left(\frac{\log T}{T}\right)^{\frac{1}{2+d}} \bigg] \ge\delta, \qquad\text{for all \ensuremath{T\ge T_{\delta}}},
	\]
	where $C>0$ is a constant. Consequently, for any reinforcement learning algorithm $\hat{Q}_T$
	and any sufficiently small $\varepsilon>0$, there exists an MDP problem
	such that in order to achieve 
	\[
	\P\Big[\big\Vert \hat{Q}_T-Q^{*}\big\Vert_{\infty}<\varepsilon\Big]\ge 1-\delta,
	\]
	one must have 
	\[
	T\ge C'd\left(\frac{1}{\varepsilon}\right)^{2+d}\log\left(\frac{1}{\varepsilon}\right),
	\]
	where $C'>0$ is a constant.
\end{thm}

\section{Conclusions}
\label{sec:conclusion}

In this paper, we considered the reinforcement learning problem for infinite-horizon discounted MDPs with a continuous state space. We focused on a reinforcement learning algorithm NNQL that is based on kernelized nearest neighbor regression. We established nearly tight finite-sample convergence guarantees showing that NNQL can accurately estimate optimal Q function using nearly optimal number of samples. In particular, our results state that the sample, space and computational complexities of NNQL scale polynomially (sometimes linearly) with the covering number of the state space, which is continuous and has uncountably infinite cardinality.

In this work, the sample complexity analysis with respect to the accuracy parameter is nearly optimal. But its dependence on the other problem parameters is not
optimized. This will be an important direction for future work. It is also interesting to generalize approach to the setting of MDP beyond infinite horizon discounted problems, such as finite horizon or average-cost problems.  Another possible direction for future work is to combine NNQL with a smart exploration policy, which may further improve the performance of NNQL. It would also be of much interest to investigate whether our approach, specifically the idea of using nearest neighbor regression, can be extended to handle infinite or even continuous action spaces.

\section*{Acknowledgment}

This work is supported in parts by NSF projects CNS-1523546, CMMI-1462158 and CMMI-1634259 and MURI 133668-5079809.

\bibliographystyle{plain}
\bibliography{NNQL}

\newpage

\appendix

\section{Related works}\label{sec:related}

Given the large body of relevant literature, even surveying the work on Q-learning in a satisfactory manner is beyond the scope of this paper. 
Here we only mention the most relevant prior works, and compare them to ours in terms of the assumptions needed, the algorithmic approaches considered, 
and the performance guarantees provided. Table~\ref{table:related-work} provides key representative works from the literature and contrasts them with our result. 

Q-learning has been studied extensively for finite-state MDPs. \cite{Tsitsiklis1994Qlearning} and~\cite{Jaakkola1994Qlearning} 
are amongst the first to establish its asymptotic convergence. Both of them cast Q-learning as a stochastic approximation scheme --- we utilize 
this abstraction as well. More recent work studies non-asymptotic performance of Q-learning; 
see, e.g.,~\cite{Szepesy1997Qlearning}, \cite{Even-Dar2004Qlearning}, and~\cite{Lim2005finite}. 
Many variants of Q-learning have also been proposed and analyzed, including Double Q-learning~\cite{Hasselt2010DoubleQL}, 
Speedy Q-learning~\cite{Azar2011SQL}, Phased Q-learning~\cite{Kearns1999PQL} and  Delayed Q-learning~\cite{Stehl2006PAC}.

A standard approach for continuous-state MDPs with known transition kernels, is to construct a reduced model by discretizing state space and show that the new finite MDP approximates the original one. For example, Chow and Tsitsiklis establish approximation guarantees for a multigrid algorithm when the state space is compact~\cite{Chow1989Complexity,Chow1991TACmultigrid}. This result is recently extended to average-cost problems and to general Borel state and action spaces in~\cite{Saldi2017}. To reduce the computational complexity, Rust proposes a randomized version of the multigrid algorithm and provides a bound on its approximation accuracy~\cite{Rust1997}. Our approach bears some similarities to this line of work: we also use state space discretization, and impose similar continuity assumptions on the MDP model. However, we do not require the transition kernel to be known, nor do we construct a reduced model; rather, we learn the action-value function of the original MDP directly by observing its sample path. 

The closest work to this paper is by Szepesvari and Smart~\cite{szepesvari2004interpolation}, wherein they consider a variant of Q-learning combined with local function approximation methods. The algorithm approximates the optimal Q-values at a given set of sample points and interpolates it for each query point. Follow-up work considers combining Q-learning with linear function approximation~\cite{melo2007qlearning}. Despite algorithmic similarity, their results are distinct from ours: they establish \emph{asymptotic} convergence of the algorithm, based on the assumption that the data-sampling policy is stochastic stationary. In contrast, we provide finite-sample bounds, and our results apply for arbitrary sample paths (including non-stationary policies). Consequently, our analytical techniques are also different from theirs.  

Some other closely related work is by Ormoneit and coauthors on model-free reinforcement learning for continuous state  with unknown transition kernels~\cite{Ormoneit2002Discounted,Ormoneit2002Average}.  Their approach, called KBRL,  constructs a kernel-based approximation of the {conditional expectation} that appears in the Bellman operator. 
Value iteration can then be run using the approximate Bellman operator, and asymptotic consistency is established for the resulting fixed points. A 
subsequent work demonstrates applicability of KBRL to practical large-scale problems~\cite{barreto2016practical}. 
Unlike our approach, KBRL is an offline, batch algorithm in which data is sampled at once and remains the same 
 throughout the iterations of the algorithm. Moreover, the aforementioned work does not provide convergence 
 rate or finite-sample performance guarantee for KBRL. The idea of approximating the Bellman operator 
by an empirical estimate, has also been used  in the context of discrete state-space problems~\cite{Haskell2016EDP}. 
The approximate operator is used to develop Empirical Dynamic Programming (EDP) algorithms including value and policy iterations, 
for which non-asymptotic error bounds are provided. EDP is again an offline batch algorithm; moreover, it requires multiple, 
independent transitions to be sampled for each state, and hence does not apply to our setting with a single sample path. 

In terms of theoretical results, most relevant is the work in~\cite{Munos2008finite}, who also obtain finite-sample performance guarantees for continuous space problems with unknown transition kernels. Extension to the setting with a single sample path is considered in~\cite{Antos2008single}. The algorithms considered therein, including fitted value iteration and Bellman-residual minimization based fitted policy iteration, are different from ours. In particular, these algorithms perform updates in a batch fashion and require storage of all the data throughout the iterations.

There are other papers that provide finite-sample guarantees, such as~ \cite{liu2015TD, dalal2017TD}; however, their settings (availability of i.i.d. data), algorithms (TD learning) and proof techniques are very different from ours. The work by Bhandari et al. ~\cite{bhandari2018TD} also provides a finite sample analysis of TD learning with linear function approximation, for both the i.i.d. data model and a single trajectory. We also note that the work on PAC-MDP methods~\cite{pazis2013pac} explores the impact of exploration policy on learning performance. The focus of our work is estimation of Q-functions rather than the problem of exploration; nevertheless, we believe it is an interesting future direction to study combining our algorithm with smart exploration strategies.

\section{Proof of Lemma~\ref{lem:Q_Lip}}
\label{sec:proof_Q_Lip}

\begin{proof}
	Let $\mathcal{D}$ be the set of all functions $u:\mX\times\mA\to\real$
	such that $\left\Vert u\right\Vert _{\infty}\le V_{\max}.$ 
	Let $\mathcal{L}$ be the set of all functions $u:\mX\times\mA\to\real$
	such that $u(\cdot,a)\in\text{Lip}(\mX,M_{r}+\gamma V_{\max}M_{p}),\forall a\in\mA$.
	Take any
	$u\in\mathcal{D}$, and fix an arbitrary $a\in\mA$. For any $x\in\mX$,
	we have 
	\begin{align*}
	\left|(Fu)(x,a)\right| =\left|r(x,a)+\gamma\E\left[\max_{b\in\mA}u(Z,b)|x,a\right]\right| \le R_{\max}+\gamma V_{\max} =V_{\max},
	\end{align*}
	where the last equality follows from the definition of $V_{\max}$. This
	means that $Fu\in\mathcal{D}$. Also, for any $x,y\in\mX$, we have
	\begin{align*}
	\left|(Fu)(x,a)-(Fu)(y,a)\right| & =\left|r(x,a)-r(y,a)+\gamma\E\left[\max_{b\in\mA}u(Z,b)|x,a\right]-\gamma\E\left[\max_{b\in\mA}u(Z,b)|y,a\right]\right|\\
	& \le\left|r(x,a)-r(y,a)\right|+\gamma\left|\int_{\mX}\max_{b\in\mA}u(z,b)\left(p(z|x,a)-p(z|y,a)\right)\lambda(dz)\right|\\
	& \le M_{r}\rho(x,y)+\gamma\int_{\mX}\left|\max_{b\in\mA}u(z,b)\left(p(z|x,a)-p(z|y,a)\right)\right|\lambda(dz)\\
	& \le M_{r}\rho(x,y)+\gamma\left\Vert u\right\Vert _{\infty}\cdot\int_{\mX}\left|p(z|x,a)-p(z|y,a)\right|\lambda(dz)\\
	& \le\left[M_{r}+\gamma V_{\max}M_{p}\right]\rho(x,y).
	\end{align*}
	This means that $(Fu)(\cdot,a)\in\text{Lip}\left(\mX,M_{r}+\gamma V_{\max}M_{p}\right)$,
	so $Fu\in\mathcal{L}$. Putting together, we see that $F$ maps $\mathcal{D}$
	to $\mathcal{D}\cap\mathcal{L}$, which in particular implies that $F$
	maps $\mathcal{D}\cap\mathcal{L}$ to itself. Since $\mathcal{D}\cap\mathcal{L}$
	is closed and $F$ is $\gamma$-contraction, both with respect to
	$\left\Vert \cdot\right\Vert _{\infty}$, the Banach fixed point theorem
	guarantees that $F$ has a unique fixed point $Q^{*}\in\mathcal{D}\cap\mathcal{L}$.
	This completes the proof of the lemma.
\end{proof}

\section{Examples of Nearest Neighbor Regression Methods}\label{sec:NN_examples}

Below we describe three representative nearest neighbor regression methods, each of which corresponds to a certain choice of the kernel function $ K $ in the averaging procedure~\eqref{eq:NN_average}.

\begin{itemize}
	\item {\textbf{$k$-nearest neighbor ($k$-NN) regression}}: For each $x\in \mX,$ we find its $k$ nearest neighbors in the subset $\mX_h$ and average their Q-values, where  $k\in [n] $ is a pre-specified number. Formally, let $c_{(i)}(x)$ denote the $i$-th closest data point to $ x $ amongst the set $\mX_h.$ Thus, the distance of each state in $\mX_h$ to $x$ satisfies
	$
	\rho(x,c_{(1)}(x)) \leq \rho(x,c_{(2)}(x)) \leq \cdots\leq \rho(x,c_{(n)}(x)).
	$
	Then the $k$-NN estimate for the Q-value of $(x,a)$ is given by 
	$
	(\Gamma_{\text{NN}}q)(x,a) = \frac{1}{k} \sum_{i=1}^{k} q\big( c_{(i)}(x),a \big).
	$
	This corresponds to using in \eqref{eq:NN_average} the following weighting function
	\[
	K(x,c_{i})=\frac{1}{k}\indic \big\{ \rho(x,c_{i})\leq\rho(x,c_{(k)}(x)) \big\}.
	\]
	Under the definition of $ \mX_{h} $ in Section~\ref{sec:discretization},  the assumption~\eqref{eq:weight_assumption} is satisfied if we use $ k=1 $. For other values of $ k $, the assumption holds with a potentially different value of $ h $.
	
	\item {\textbf{Fixed-radius near neighbor regression}}: We find all neighbors of $x$ up to a threshold distance $h>0$ and average their Q-values. The definition of $ \mX_{h} $ ensures that at least one point $c_{i}\in \mX_{h}$ is within the threshold distance $h$, i.e., $\forall x\in\mX,$ $\exists c_{i}\in\mX_{h}$ such that $\rho(x,c_{i})\leq h$. We then can define the weighting function function according to 
	\[
	K\left(x,c_{i}\right)=\frac{\indic\{\rho(x,c_{i})\leq h\}}{\sum_{j=1}^{n}\indic\{\rho(x,c_{j})\leq h\}}.
	\]
	
	\item {\textbf{Kernel regression}}: Here the Q-values of the neighbors of $ x $ are averaged in a \emph{weighted} fashion according to some kernel function~\cite{Nadaraya1964regression, Watson1964regression}. The kernel function $\phi:\real^+ \rightarrow[0,1]$ takes as input a distance (normalized by the bandwidth parameter $h$) and outputs a similarity score between $0$ and $1.$ Then the weighting function $K(x,c_{i})$ is given by
	\[
	K(x,c_{i})=\frac{\phi\left(\frac{\rho(x,c_{i})}{h}\right)}{\sum_{j=1}^n\phi\left(\frac{\rho(x,c_{j})}{h}\right)}.
	\]
	For example, a (truncated) Gaussian kernel corresponds to $\phi(s)=\exp\left(-\frac{s^2}{2}\right) \indic\{s \leq 1\}$. Choosing $\phi(s)=\indic\{s\leq 1\}$ reduces to the fixed-radius NN regression described above.
	
\end{itemize}

\section{Bounds on Covering time}\label{sec:proof_cover_time}

\subsection{Proof of Proposition~\ref{prop:cover_time}}

\begin{proof} 
		Without loss of generality, we may assume that the balls $\{\ball_{i},i\in[n]\}$
	are disjoint, since the covering time will only become smaller if
	they overlap with each other. Note that under $\varepsilon$-greedy
	policy, equation~\eqref{eq:unif.erg} implies that $\forall t\geq0,$ $\forall x\in\mX,$ $\forall a\in\mA,$
	\begin{align}
	\Pr\left(x_{m+t}\in\cdot,a_{m+t}=a|x_{t}=x\right)\geq\frac{\varepsilon}{|\mA|}\psi\nu(\cdot).\label{eq:ergodic_eps}
	\end{align}	
	First assume that the above assumption holds with $m=1$. Let $M\triangleq n\left|\mA\right|$
	be the total number of ball-action pairs. Let $(\mathcal{P}_{1},\ldots,\mathcal{P}_{M})$
	be a fixed ordering of these $M$ pairs. For each integer $t\ge1$,
	let $K_{t}$ be the number of ball-action pairs visited up to time
	$t$. Let $T\triangleq\inf\left\{ t\ge1:K_{t}=M\right\} $ be the
	first time when all ball-action pairs are visited. For each $k\in\{1,2,\ldots,M\}$,
	let $T_{k}\triangleq\left\{ t\ge1:K_{t}=k\right\} $ be the the first
	time when $k$ pairs are visited, and let $D_{k}\triangleq T_{k}-T_{k-1}$
	be the time to visit the $k$-th pair after $k-1$ pairs have been
	visited.  We use the convention that $T_{0}=D_{0}=0$. By definition,
	we have $T=\sum_{k=1}^{M}D_{k}.$ 
	
	When $k-1$ pairs have been visited, the probability of visiting a
	\emph{new} pair is at least 
	\begin{align*}
		&\min_{I\subseteq[M],\left|I\right|=M-k+1}{\textstyle \Pr}\bigg( (x_{T_{k-1}+1},a_{T_{k-1}+1})\in\bigcup_{i\in I}\mathcal{P}_{i}|x_{T_{k-1}}\bigg)\\ =& \min_{I\subseteq[M],\left|I\right|=M-k+1}\sum_{i\in I}{\textstyle \Pr}\left( (x_{T_{k-1}+1},a_{T_{k-1}+1})\in\mathcal{P}_{i}|x_{T_{k-1}}\right) \\
		\ge& (M-k+1)\min_{i\in[M]}{\textstyle \Pr_{\pi}}\left( (x_{T_{k-1}+1},a_{T_{k-1}+1})\in\mathcal{P}_{i}|x_{T_{k-1}}\right) \\
		\ge& (M-k+1)\cdot\varphi\nu_{\min}\cdot\frac{\varepsilon}{\left|\mA\right|},
	\end{align*}
	where the last inequality follows from Eq.~(\ref{eq:ergodic_eps}).
	Therefore, $D_{k}$ is stochastically dominated by a geometric random
	variable with mean at most $\frac{\left|\mA\right|}{(M-k+1)\varepsilon\varphi\nu_{\min}}.$
	It follows that 
	\begin{align*}
		\E T & =\sum_{k=1}^{M}\E D_{k} \le\sum_{k=1}^{M}\frac{\left|\mA\right|}{(M-k+1)\varepsilon\varphi\nu_{\min}} =O\left(\frac{\left|\mA\right|}{\varepsilon\varphi\nu_{\min}}\log M\right).
	\end{align*}
	This prove the proposition for $m=1$. 
	
	For general values of $m$, the proposition follows from a similar argument
	by considering the MDP only at times $t=m,2m,3m,\ldots.$ 
\end{proof} 

\subsection{Proof of Proposition~\ref{prop:cover_time_single}}

\begin{proof}
We shall use a line of argument similar to that in the proof of Proposition~\ref{prop:cover_time}. We assume that the balls $\{\ball_{i},i\in[n]\}$
are disjoint. 
Note that under $\varepsilon$-greedy policy $\pi$, for all $t\geq0,$ for all $x\in\mX,$ we have
\begin{equation}
   \textstyle{\Pr} _\pi\left(a_t=\hat{a}_1,\ldots,a_{t+m-1}=\hat{a}_m|x_{t}=x\right)\geq\left(\frac{\varepsilon}{|\mA|}\right)^m.\label{eq:eps_special_sequence}
\end{equation}	
The equation~(\ref{eq:unif.erg.single}) implies that 
\begin{align*}
&\textstyle{\Pr} _\pi\left(x_{t+m}\in \cdot |x_t=x\right) \\
\geq & \textstyle{\Pr}\left(x_{t+m}\in \cdot |x_t=x, a_t=\hat{a}_1,\ldots,a_{t+m-1}=\hat{a}_m\right) \times \textstyle{\Pr}_{\pi}
\left(a_t=\hat{a}_1,\ldots,a_{t+m-1}=\hat{a}_m|x_{t}=x\right)\\
\geq& \psi \nu(\cdot) \left(\frac{\varepsilon}{|\mA|}\right)^m.
\end{align*}
Thus for each $a\in\mA,$
 \begin{align} \label{eq:ergodic_eps_single}
 \textstyle{\Pr} _\pi\left(x_{t+m}\in \cdot, a_{t+m}=a |x_t=x\right) \geq  \psi \nu(\cdot) \left(\frac{\varepsilon}{|\mA|}\right)^{m+1}. 
 \end{align}

We first consider the case $m=1$ and use the same notation as in the proof of Proposition~\ref{prop:cover_time}. 
	When $k-1$ pairs have been visited, the probability of visiting a
	\emph{new} pair is at least 
		\begin{align*}
	&\min_{I\subseteq[M],\left|I\right|=M-k+1}{\textstyle \Pr}\bigg( (x_{T_{k-1}+1},a_{T_{k-1}+1})\in\bigcup_{i\in I}\mathcal{P}_{i}|x_{T_{k-1}}\bigg)\\  
	\ge& (M-k+1)\min_{i\in[M]}{\textstyle \Pr_{\pi}}\left( (x_{T_{k-1}+1},a_{T_{k-1}+1})\in\mathcal{P}_{i}|x_{T_{k-1}}\right) \\
	\ge& (M-k+1)\cdot\varphi\nu_{\min}\cdot\left(\frac{\varepsilon}{|\mA|}\right)^{2},
	\end{align*}
	where the last inequality follows from Eq.~(\ref{eq:ergodic_eps_single}).
	Therefore, $D_{k},$ the time to visit the $k$-th pair after $k-1$ pairs have been visited, is stochastically dominated by a geometric random
	variable with mean at most $\frac{\left(\left|\mA\right|/\varepsilon\right)^2}{(M-k+1)\varphi\nu_{\min}}.$
	It follows that 
	\begin{align*}
	\E T & =\sum_{k=1}^{M}\E D_{k} \le\sum_{k=1}^{M}\frac{\left(\left|\mA\right|/\varepsilon\right)^2}{(M-k+1)\varphi\nu_{\min}}  =O\left(\frac{\left(\left|\mA\right|/\varepsilon\right)^2}{\varphi\nu_{\min}}\log M\right).
	\end{align*}
	This prove the proposition for $m=1$. 
	
	For general values of $m$, the proposition follows from a similar argument
	by considering the MDP only at times $t=m,2m,3m,\ldots.$ 
\end{proof}

\section{Proof of the Main Result: Theorem~\ref{thm:NN Q-learning}}
\label{sec:proof_main}

The proof of Theorem~\ref{thm:NN Q-learning} consists of three key steps summarized as follows. 

{\bf Step 1. Stochastic Approximation.} Since the nearest-neighbor approximation of the Bellman operator induces a biased update for $q^k$ at each step, the key step in our proof is to analyze a Stochastic Approximation (SA) algorithm with \emph{biased} noisy updates. In particular, we establish its finite-sample convergence rate in Theorem~\ref{thm:PAC_linear_inf}, which does not follow from available convergence theory. This result itself may be of independent interest. 

{\bf Step 2. Properties of NNQL.} To apply the stochastic approximation result to NNQL, we need to characterize some key properties of NNQL, including (i) the stability of the algorithm (i.e., the sequence $q^k$ stays bounded), as established in Lemma~\ref{lem:stability}; (ii) the contraction property of the joint Bellman-NN operator, as established in Lemma~\ref{lem:G-contraction}; and (iii)  the error bound induced by discretization of the state space, as established in Lemma~\ref{lem:error_discrete}. 

{\bf Step 3. Apply SA to NNQL.} We apply the stochastic approximation result to establish the finite-sample convergence of NNQL. In particular, step 2 above ensures that NNQL satisfies the assumptions in Theorem~\ref{thm:PAC_linear_inf}. Applying this theorem, we prove
that NNQL converges to a neighborhood of $q_{h}^* $, the fixed point of the Joint Bellman-NN operator $ G $, after a sufficient number of iterations. The proof of Theorem~\ref{thm:NN Q-learning} is completed by relating $ q_{h}^* $ to the true optimal  Q-function $ Q^* $, and by bounding the number of time steps in terms of the the number of iterations and the covering time.

\subsection{Stochastic Approximation}\label{subsec:SA}

Consider a generic iterative stochastic approximation algorithm, where the iterative update rule is has the following form: let 
$\theta^t$ denote the {\em state} at time $t$, then it is updated as 
\begin{equation}
\theta^{t+1}=\theta^{t}+\alpha_{t}\left(F(\theta^{t})-\theta^{t}+w^{t+1}\right),\label{eq:SA_fixed_point}
\end{equation}
where $\alpha_{t}\in[0,1]$ is a step-size parameter, $w^{t+1}$ is
a noise term and $F$ is the functional update of interest. 
\begin{thm}
	\label{thm:PAC_linear_inf}Suppose that the mapping $F:\real^{d}\to\real^{d}$
	has a unique fixed point $\theta^{*}$ with $\|\theta^*\|_\infty\leq V,$ and is a $\gamma$-contraction
	with respect to the $\ell_{\infty}$ norm in the sense that 
	\begin{align*}
		\left\Vert F(\theta)-F(\theta')\right\Vert _{\infty} & \le\gamma\left\Vert \theta-\theta'\right\Vert _{\infty} 
	\end{align*}
	for all $\theta,\theta'\in\real^{d}$, where $0<\gamma<1$.  Let $\{\mathcal{F}^{t}\}$ be an increasing sequence of $\sigma$-fields so
	that $\alpha_{t}$ and $w^{t}$ are $\mathcal{F}^{t}$-measurable random variables, and $\theta^{t}$ be updated as per \eqref{eq:SA_fixed_point}. 
	Let $\delta_{1},\delta_{2}, M, V$
	be non-negative deterministic constants. Suppose that the following
	hold with probability $1$:
	\begin{enumerate}
		\item \label{a:1} The bias $\Delta^{t+1}=\E\left[w^{t+1}\,|\,\mathcal{F}^{t}\right]$
		satisfies $\left\Vert \Delta^{t+1}\right\Vert _{\infty}\le\delta_{1}+\delta_{2}\left\Vert \theta^{t}\right\Vert _{\infty}$, for all $t \geq 0$;
		\item  \label{a:3} $\left\Vert w^{t+1}-\Delta^{t+1}\right\Vert _{\infty}\leq M,$ for all $t\geq 0$;
		\item  \label{a:4} $\left\Vert \theta^{t}\right\Vert _{\infty}\leq V,$  for all $t\geq 0$.
	\end{enumerate}
	Further, we choose 
	\begin{equation}
	\alpha_{t}=\frac{\beta}{\beta+t},\label{eq:linear-learning-rate}
	\end{equation}
	where $\beta=\frac{1}{1-\gamma}$.
	Then for each $0<\varepsilon<\min\{2V\beta, 2M \beta^2\}$,
	after 
	\[
	T=\frac{48VM^{2}\beta^{4}}{\varepsilon^{3}}\log\left(\frac{32dM^{2}\beta^{4}}{\delta\varepsilon^{2}}\right)+\frac{6V(\beta-1)}{\varepsilon}
	\]
	iterations of (\ref{eq:SA_fixed_point}), with probability at least
	$1-\delta$, we have 
	\[
	\left\Vert \theta^{T}-\theta^{*}\right\Vert _{\infty}\leq\beta(\delta_{1}+\delta_{2}V)+\varepsilon.
	\]
\end{thm}
\begin{proof}
	We define two auxiliary sequences: for $i \in [d]$, let $u_{i}^{0}=\theta_{i}^{0}$,
	$r_{i}^{0}=0$ and 
	\begin{align*}
		u_{i}^{t+1} & =(1-\alpha_{t})u_{i}^{t}+\alpha_{t}\underbrace{(w_{i}^{t+1}-\Delta_{i}^{t+1})}_{\bar{w}_{i}^{t+1}},\\
		r_{i}^{t+1} & =\theta_{i}^{t+1}-u_{i}^{t+1}.
	\end{align*}
	By construction, $\theta^{t}=u^{t}+r^{t}$ for all $t$.  We first analyze the convergence rate of the $(u^{t})$ sequence. One has
	\begin{align*}
		u_{i}^{t+1} & =(1-\alpha_{t})u_{i}^{t}+\alpha_{t}\bar{w}_{i}^{t+1}\\
		& =(1-\alpha_{t})(1-\alpha_{t-1})u_{i}^{t-1}+(1-\alpha_{t})\alpha_{t-1}\bar{w}_{i}^{t}+\alpha_{t}\bar{w}_{i}^{t+1}\\
		& =\sum_{j=1}^{t+1}\eta^{t+1,j}\bar{w}_{i}^{j},
	\end{align*}
	where we define
	\[
	\eta^{t+1,j} :=\alpha_{j-1}\cdot\prod_{l=j}^{t}(1-\alpha_{l}).
	\]
	Note that the centered noise $\bar{w}_{i}^{t+1}:=w_{i}^{t+1}-\Delta_{i}^{t+1}$
	satisfies 
	\begin{align}
		\E\left[\bar{w}_{i}^{t+1}|\mathcal{F}^{t}\right] & =\E\left[w_{i}^{t+1}|\mathcal{F}^{t}\right]-\Delta_{i}^{t+1}=0,\nonumber \\
		\E\left[\left|\bar{w}_{i}^{t+1}\right||\mathcal{F}^{t}\right] & =\E\left[\left|w_{i}^{t+1}-\Delta_{i}^{t+1}\right||\mathcal{F}^{t}\right]\leq M.\label{eq:bound_centered_noise}
	\end{align}
	Now
	\begin{align}\label{eq:mg1}
		\E\left[\eta^{t+1,j}\bar{w}_{i}^{j}\,|\,\mathcal{F}^{j-1}\right] & =\eta^{t+1,j}\E\left[\bar{w}_{i}^{j}\,|\,\mathcal{F}^{j-1}\right] ~= 0. 
	\end{align}
	With the linear learning rate defined in Eq. \eqref{eq:linear-learning-rate}, and the fact that
	$\beta=\frac{1}{1-\gamma}>1,$ $\forall j\in[1,t+1],$
	we have 
	\begin{align}\label{eq:mg2}
		\eta^{t+1,j} & =\frac{\beta}{j-1+\beta}\cdot\prod_{l=j}^{t}\frac{l}{l+\beta}<\frac{\beta}{j}\prod_{l=j}^{t}\frac{l}{l+1}=\frac{\beta}{t+1}.
	\end{align}
	Since the centered noise sequence $\{\bar{w}_{i}^{1},\bar{w}_{i}^{2},\ldots,\bar{w}_{i}^{t+1}\}$
	is uniformly bounded by $M>0$, it follows that
	\begin{align}\label{eq:mg3}
		\left|\eta^{t+1,j}\bar{w}_{i}^{j}\right| & \leq\frac{M\beta}{t+1}.
	\end{align}
	Define, for $1\leq s \leq t+1$, 
	\begin{align}
		z^{t+1, i}_s & := \sum_{j=1}^{s}\eta^{t+1,j}\bar{w}_{i}^{j}, 
	\end{align}
	and $z^{t+1, i}_0 = 0$. Then it follows that 
	\begin{align}
		\E\big[ z^{t+1, i}_{s+1} | \mathcal{F}^s \big] & = z^{t+1, i}_s.
	\end{align}
	And from \eqref{eq:mg1}-\eqref{eq:mg3}, it follows that
	\begin{align}\label{eq:mg4}
		| z^{t+1, i}_{s+1} - z^{t+1, i}_s| & \leq \frac{M\beta}{t+1}. 
	\end{align}
	That is, $z^{t+1, i}_s$ is a Martingale with bounded differences. And $u^{t+1}_i = z^{t+1, i}_{t+1}$. 
	This, using Azuma-Hoeffding's inequality, will provide us desired bound on $|u^{t+1}_i|$. To that end,
	let us recall the Azuma-Hoeffding's inequality.
	\begin{lem}[Azuma-Hoeffding]
		Let $X_j$ be Martingale with respect to filtration $\mathcal{F}_j$, i.e. $\E[X_{j+1} | \mathcal{F}_j] = X_j$ for $j \geq 1$ with $X_0 = 0$. 
		Further, let $|X_{j} - X_{j-1}| \leq c_j$ with probability $1$ for all $j \geq 1$. Then  
		for all $ \varepsilon \ge 0 $,
		\begin{align*}
			\Pr \left[ \left| X_n \right| \ge \varepsilon \right] \le 2 \exp \left( -\frac{ \varepsilon^2}{2\sum_{j=1}^n c_j^2} \right).
		\end{align*}
	\end{lem}
	Applying the lemma to $z^{t+1, i}_j$ for $j \geq 0$ with $z^{t+1, i}_0 = 0$, \eqref{eq:mg4} and the 
	fact that $u^{t+1}_i = z^{t+1, i}_{t+1}$, we obtain that
	\begin{align}
		\P\left(\left|u_{i}^{t+1}\right|>\varepsilon\right) & \leq 2\exp\left(-\frac{(t+1)\varepsilon^{2}}{2M^{2}\beta^{2}}\right).
	\end{align}
	Therefore, by union bound we obtain 
	\begin{align*}
		\P\left(\text{\ensuremath{\exists}}t\geq T_{1}\text{ such that}\left|u_{i}^{t}\right|>\varepsilon\right) & \leq\sum_{t=T_{1}}^{\infty}\P\left(\left|u_{i}^{t}\right|>\varepsilon\right)\\
		& \leq2\sum_{t=T_{1}}^{\infty}\exp\left(-\frac{t\varepsilon^{2}}{2M^{2}\beta^{2}}\right)\\
		& =\frac{2\exp\left(-\frac{T_{1}\varepsilon^{2}}{2M^{2}\beta^{2}}\right)}{1-\exp\left(-\frac{\varepsilon^{2}}{2M^{2}\beta^{2}}\right)} &  & \\
		& \leq\frac{8M^{2}\beta^{2}}{\varepsilon^{2}}\exp\left(-\frac{T_{1}\varepsilon^{2}}{2M^{2}\beta^{2}}\right), &  &  
	\end{align*}
	where the last step follows from the fact that $e^{-x} \leq 1-\frac{x}{2}$ for $0 \leq x \leq \frac12$, and $\varepsilon \leq M \beta$.
	By a union bound over all $i\in[d],$ we deduce that 
	\begin{align}
		\P\left(\text{\ensuremath{\exists}}t\geq T_{1}\text{ such that}\left\Vert u^{t}\right\Vert _{\infty}>\varepsilon\right) & \leq\sum_{i\in[d]}\P\left(\text{\ensuremath{\exists}}t\geq T_{1}\text{ such that}\left|u_{i}^{t}\right|>\varepsilon\right)\notag\\
		& \leq\frac{8dM^{2}\beta^{2}}{\varepsilon^{2}}\exp\left(-\frac{T_{1}\varepsilon^{2}}{2M^{2}\beta^{2}}\right). \label{eq:Pr_error_ut}
	\end{align}
	
	Next we focus on the residual sequence $(r^{t})$. Assume that $\forall t\ge T_{1},$
	$\left\Vert u^{t}\right\Vert _{\infty}\leq\varepsilon_{1}$, where  $0<\varepsilon_{1}<\min\{V, M \beta\}$.
	For each $i\in[d]$ and $t\geq T_{1}$, we get
	\begin{align*}
		& \left|r_{i}^{t+1}-\theta_{i}^{*}\right|\\
		= & \left|\theta_{i}^{t+1}-u_{i}^{t+1}-\theta_{i}^{*}\right| &  & \text{by definition}\\
		= & \left|\theta_{i}^{t}+\alpha_{t}\left(F_{i}(u^{t}+r^{t})-u_{i}^{t}-r_{i}^{t}+w_{i}^{t+1}\right)-u_{i}^{t}-\alpha_{t}\left(-u_{i}^{t}+w_{i}^{t+1}-\Delta_{i}^{t+1}\right)-\theta_{i}^{*}\right| &  & \text{by definition}\\
		= & \left|r_{i}^{t}+\alpha_{t}\left(F_{i}(u^{t}+r^{t})-r_{i}^{t}\right)-\theta_{i}^{*}+\alpha_{t}\Delta_{i}^{t+1}\right| &  & \text{rearranging}\\
		= & \left|(1-\alpha_{t})(r_{i}^{t}-\theta_{i}^{*})+\alpha_{t}\left(F_{i}(u^{t}+r^{t})-\theta_{i}^{*}\right)+\alpha_{t}\Delta_{i}^{t+1}\right| &  & \text{rearranging}\\
		\le & (1-\alpha_{t})\left|r_{i}^{t}-\theta_{i}^{*}\right|+\alpha_{t}\gamma\left\Vert u^{t}+r^{t}-\theta^{*}\right\Vert _{\infty}+\alpha_{t}\left\Vert \Delta^{t+1}\right\Vert _{\infty} &  & \text{\ensuremath{F} is \ensuremath{\gamma}-contraction}\\
		\le & (1-\alpha_{t})\left|r_{i}^{t}-\theta_{i}^{*}\right|+\alpha_{t}\gamma\left\Vert r^{t}-\theta^{*}\right\Vert _{\infty}+\alpha_{t}\gamma\varepsilon_{1}+\alpha_{t}\left(\delta_{1}+\delta_{2}\left\Vert \theta^{t}\right\Vert _{\infty}\right) &  & \left\Vert u^{t}\right\Vert _{\infty}\le\varepsilon_{1},\forall t\ge T_{1}\\
		\le & (1-\alpha_{t})\left|r_{i}^{t}-\theta_{i}^{*}\right|+\alpha_{t}\gamma\left\Vert r^{t}-\theta^{*}\right\Vert _{\infty}+\alpha_{t}\left(\gamma\varepsilon_{1}+\delta_{1}+\delta_{2}V\right) &  & \left\Vert \theta^{t}\right\Vert _{\infty}\leq V
	\end{align*}
	Taking the maximum over $i\in[d]$ on both sides, we obtain
	\begin{align*}
		\left\Vert r^{t+1}-\theta^{*}\right\Vert _{\infty} & \le(1-\alpha_{t})\left\Vert r^{t}-\theta^{*}\right\Vert _{\infty}+\alpha_{t}\gamma\left\Vert r^{t}-\theta^{*}\right\Vert _{\infty}+\alpha_{t}\left(\gamma\varepsilon_{1}+\delta_{1}+\delta_{2}V\right)\\
		& =\Big(1-\underbrace{(1-\gamma)}_{\frac{1}{\beta}}\alpha_{t}\Big)\underbrace{\left\Vert r^{t}-\theta^{*}\right\Vert _{\infty}}_{D_{t}}+\alpha_{t}\underbrace{\left(\gamma\varepsilon_{1}+\delta_{1}+\delta_{2}V\right)}_{H},\qquad\forall t\ge T_{1}.
	\end{align*}
	For any $\varepsilon_{2}>0$,  
	we will show that after at most 
	\[
	T_{2} \triangleq \frac{3V (T_{1} + \beta -1)}{\varepsilon_{2}}
	\]
	iterations, we have
	\[
	\left\Vert r^{T_{2}}-\theta^{*}\right\Vert _{\infty}\leq H\beta+\varepsilon_{2}.
	\]
	If for some $T \in [T_{1}, \infty) $ there holds $D_{T}\leq H\beta+\varepsilon_{2}$,
	then we have 
	\begin{align*}
		D_{T+1} & \leq\left(1-\frac{\alpha_{t}}{\beta}\right)\left(H\beta+\varepsilon_{2}\right)+\alpha_{t}H &  & \alpha_{t}\leq\beta\\
		& \leq H\beta+\varepsilon_{2}
	\end{align*}
	Indeed by induction, we have 
	\[
	D_{t}\leq H\beta+\varepsilon_{2},\quad\forall t\geq T.
	\]
	Let $\widehat{T} \triangleq \sup\left\{ t\geq T_{1}:D_{t}>H\beta+\varepsilon_{2}\right\} $
	be the last time that $D_{t}$ exceeds $H\beta+\varepsilon_{2}$. For
	each $T_{1}\leq t\leq \widehat{T},$ the above argument implies that we must have $D_{t}>H\beta+\varepsilon_{2}$.
	We can rewrite the iteration for $D_{\widehat{T}}$ as follows: 
	\begin{align*}
		D_{\widehat{T}}-H\beta & \leq\left(D_{\widehat{T}-1}-H\beta\right)\left(1-\frac{\alpha_{\widehat{T}-1}}{\beta}\right)\\
		& \leq\left(D_{T_{1}}-H\beta\right)\prod_{j=T_{1}}^{\widehat{T}-1}\left(1-\frac{\alpha_{j}}{\beta}\right) &  & \text{Iteration, }D_{t}-H\beta>\varepsilon_{2}>0\\
		& =\left(D_{T_{1}}-H\beta\right)\frac{T_{1}+\beta-1}{\widehat{T}+\beta-1} &  & \alpha_{j}=\frac{\beta}{j+\beta}.
	\end{align*}
	But we have the bound 
	\begin{align*}
		D_{T_1}  &= \| r^{T_1} - \theta^* \|_\infty \\
		&= \| \theta^{T_1} - u^{T_1} - \theta^* \|_\infty \\
		&\le \| \theta^{T_1}\|_\infty + \| u^{T_1} \|_\infty +\| \theta^* \|_\infty \\
		&\le 3V,
	\end{align*}
	where the last step holds because $\|\theta^{T_1}\|_\infty\leq V,$  $\|\theta^*\|\leq V $ and $ \| u^{T_1}\| \le \varepsilon_1 \le V $ by assumption.
	It follows that 
	\[
	D_{\widehat{T}}-H\beta \leq\frac{3V(T_{1}+\beta-1)}{\widehat{T}}.
	\]

	Using the fact that $\varepsilon_{2}\leq D_{\widehat{T}}-H\beta,$ we get that 
	\[
	\widehat{T} \leq T_2 = \frac{3V(T_{1}+\beta-1)}{\varepsilon_{2}}.
	\]
	Therefore, for each $\varepsilon_{2}>0$, conditioned on the event
	\[
	\left\{ \forall t\geq T_{1},\left|u_{i}^{t}\right|\leq\varepsilon_1\right\} ,
	\]
	after at most $T_2$ iterations, we have
	\[
	\left\Vert r^{T_{2}}-\theta^{*}\right\Vert _{\infty}\leq H\beta+\varepsilon_{2}.
	\]
	
	It then follows from the relationship $\theta^{t}=u^{t}+r^{t}$ that
	\[
	\left\Vert \theta^{T_{2}}-\theta^{*}\right\Vert _{\infty}\leq\left\Vert r^{T_{2}}-\theta^{*}\right\Vert _{\infty}+\left\Vert u^{T_{2}}\right\Vert _{\infty}\leq H\beta+\varepsilon_{1}+\varepsilon_{2}=\beta(\delta_{1}+\delta_{2}V)+\beta\varepsilon_{1}+\varepsilon_{2}.
	\]
	By \eqref{eq:Pr_error_ut}, taking 
	\[
	\delta=\frac{8dM^{2}\beta^{2}}{\varepsilon_{1}^{2}}\exp\left(-\frac{T_{1}\varepsilon_{1}^{2}}{2M^{2}\beta^{2}}\right),
	\]
	i.e., 
	\[
	T_{1}=\frac{2M^{2}\beta^{2}}{\varepsilon_{1}^{2}}\log\left(\frac{8dM^{2}\beta^{2}}{\delta\varepsilon_{1}^{2}}\right),
	\]
	we are guaranteed that 
	\[
	\P\left(\forall t\geq T_{1},\left\Vert u^{t}\right\Vert _{\infty}\leq\varepsilon_1\right)\geq1-\delta.
	\]
	By setting $\varepsilon_{1}=\frac{\varepsilon}{2\beta} \leq \min\{V,M\beta\}$, 
	and $T_{2}=\frac{3V(T_{1}+\beta-1)}{\varepsilon_{2}}$, i.e., 
	\[
	T_{2}=\frac{48VM^{2}\beta^{4}}{\varepsilon^{3}}\log\left(\frac{32dM^{2}\beta^{4}}{\delta\varepsilon^{2}}\right)+\frac{6V(\beta-1)}{\varepsilon},
	\]
	we obtain that the desire result. 
\end{proof}

\subsection{Properties of NNQL} \label{subsec:properties_NNQL}

We first introduce some notations. 
Let $\mY^{k}$ be the set of all
samples drawn at iteration $k$ of the NNQL algorithms and $\mF^{k}$ be
the filtration generated by the sequence ${\mY^{0},\mY^{1},\ldots,\mY^{k-1}}.$
Thus $\{\mF^{k}\}$ is an increasing sequence of $\sigma$-fields. We denote by $\mathcal{Y}_{k}(c_{i},a)=\{Y_{t}\in \mathcal{Y}_k|Y_{t}\in\ball_{i},a_{t}=a\}$
the set of observations $ Y_t $ that fall into the neighborhood $ \ball_i $ of $ c_i $ and with action $a_t=a$ at iteration $k$. Thus the biased estimator $G^k$ (\ref{eq:Gk}) for the joint Bellman-NN operator at the end of iteration $k$ can be written as
\[
(G^{k}q)(c_{i},a)=\frac{1}{|\mathcal{Y}_{k}(c_{i},a)|}\sum_{Y_{t}\in\mathcal{Y}_{k}(c_{i},a)}\left[R_t+\gamma\max_{b\in\mA}(\Gamma_{\text{NN}}q^{k})(Y_{t+1},b)\right].
\]
The updater rule of NNQL (\ref{eq:NNQL}) can be written as 
\[
q^{k+1}(c_{i},a)=q^{k}(c_{i},a)+\alpha_k\left[(Gq^{k})(c_{i},a)-q^{k}(c_{i},a)+w^{k+1}(c_{i},a)\right],
\]
where 
\begin{align*}
	w^{k+1}(c_{i},a) & =(G^{k}q^{k})(c_{i},a)-(Gq^{k})(c_{i},a)\\
	& =\frac{1}{|\mathcal{Y}_{k}(c_{i},a)|}\sum_{Y_{t}\in\mathcal{Y}_{k}(c_{i},a)}\left[R_t+\gamma\max_{b\in\mA}(\Gamma_{\text{NN}}q^{k})(Y_{t+1},b)\right] \\
	& \qquad  -r(c_{i},a)-\gamma\E\left[\max_{b\in\mA}(\Gamma_{\text{NN}}q^{k})(x',b)\;|\;c_{i},a,\mF^{k}\right].
\end{align*}

\subsubsection{Stability of NNQL}
We first show the stability of NNQL, which is summarized in the following
Lemma.
\begin{lem}[Stability of NNQL]
	\label{lem:stability}Assume that the immediate reward is uniformly
	bounded by $R_{\max}$ and define $\beta=\frac{1}{1-\gamma}$ and
	$V_{\max}=\beta R_{\max}$. If the initial action-value function
	$q^{0}$ is uniformly bounded by $V_{\max}$, then we have 
	\[
	\left\Vert q^{k}\right\Vert _{\infty}\leq V_{\max},
	\quad\text{and}\quad
	\left|w^{k+1}(c_{i},a)-\E\left[w^{k+1}(c_i,a)\,|\,\mathcal{F}^{k}\right]\right| \leq2V_{\max},\qquad\forall k\geq0.
	\]
	
\end{lem}
\begin{proof}
	We first prove that $\left\Vert q^{k}\right\Vert _{\infty}\leq V_{\max}$
	by induction. For $k=0$, it holds by the assumption. Now assume that
	for any $0\leq\tau\leq k,$ $\left\Vert q^{\tau}\right\Vert _{\infty}\leq V_{\max}$.
	Thus 
	\begin{align*}
		&\left|q^{k+1}(c_{i},a)\right| \\
		&  =\left|q^{k}(c_{i},a)+\alpha_k\left[(G^{k}q^{k})(c_{i},a)-q^{k}(c_{i},a)\right]\right|\\
		& =\left|\left(1-\alpha_k\right)q^{k}(c_{i},a)+\frac{\alpha_k}{|\mathcal{Y}_{k}(c_{i},a)|}\sum_{Y_{t}\in\mathcal{Y}_{k}(c_{i},a)}\left[R_t+\gamma\max_{b\in\mA}(\Gamma_{\text{NN}}q^{k})(Y_{t+1},b)\right]\right|\\
		& \leq\left|\left(1-\alpha_k\right)q^{k}(c_{i},a)\right|+\frac{\alpha_k}{|\mathcal{Y}_{k}(c_{i},a)|}\sum_{Y_{t}\in\mathcal{Y}_{k}(c_{i},a)}\left(\left|R_t\right|+\gamma\left|\max_{b\in\mathcal{A}}(\Gamma_{\text{NN}}q^{k})(Y_{t+1},b)\right|\right)\\
		& =\left(1-\alpha_k\right)\left|q^{k}(c_{i},a)\right|+\frac{\alpha_k}{|\mathcal{Y}_{k}(c_{i},a)|}\sum_{Y_{t}\in\mathcal{Y}_{k}(c_{i},a)}\left(\left|R_t\right|+\gamma\max_{b\in\mA}\left|\sum_{j=1}^{n}K(Y_{t+1},c_{j})q^{k}(c_{j},b)\right|\right) \\
		& \leq\left(1-\alpha_k\right)V_{\max}+\alpha_k\left(R_{\max}+\gamma\max_{b\in\mA}\left|\sum_{j=1}^{n}K(Y_{t+1},c_{j})\right|V_{\max}\right)\\
		& =V_{\max},
	\end{align*}
	where the last equality follows from the fact that $\sum_{j=1}^{n}K(Y_{t+1},c_{j})=1.$
	Therefore, for all $k\geq0$, $\left\Vert q^{k}\right\Vert _{\infty}\leq V_{\max}.$
	The bound on $w^{k+1}$ follows from 
	\begin{align*}
		&\left|w^{k+1}(c_{i},a)-\E\left[w^{k+1}(c_i,a)\,|\,\mathcal{F}^{k}\right]\right| \\
		=&\left|(G^{k}q^{k})(c_{i},a)-(Gq^{k})(c_{i},a)-\E\left[(G^{k}q^{k})(c_{i},a)-(Gq^{k})(c_{i},a)\,|\,\mathcal{F}^{k}\right]\right|\\
		=&\left|(G^{k}q^{k})(c_{i},a)-\E\left[(G^{k}q^{k})(c_{i},a)\,|\,\mathcal{F}^{k}\right]\right|\\
		=& \Bigg|\frac{1}{|\mathcal{Y}_{k}(c_{i},a)|}\sum_{Y_{t}\in\mathcal{Y}_{k}(c_{i},a)}\Big[R_t+\gamma\max_{b\in\mA}(\Gamma_{\text{NN}}q^{k})(Y_{t+1},b)\Big] \\
		& -\E\left[\frac{1}{|\mathcal{Y}_{k}(c_{i},a)|}\sum_{Y_{t}\in\mathcal{Y}_{k}(c_{i},a)}\Big[R_t+\gamma\max_{b\in\mA}(\Gamma_{\text{NN}}q^{k})(Y_{t+1},b)\Big]\,|\,\mathcal{F}^{k}\right]\Bigg| \\
		\leq & 2R_{\max}+\frac{1}{|\mathcal{Y}_{k}(c_{i},a)|}\sum_{Y_{t}\in\mathcal{Y}_{k}(c_{i},a)}\gamma\left|\max_{b\in\mA}\sum_{j=1}^{n}K(Y_{t+1},c_{j})q^{k}(c_{j},b)\right|\\
		&~~~~+\gamma\left|\E\left[\frac{1}{|\mathcal{Y}_{k}(c_{i},a)|}\sum_{Y_{t}\in\mathcal{Y}_{k}(c_{i},a)}\max_{b\in\mA}\sum_{j=1}^{n}K(Y_{t+1},c_{j})q^{k}(c_{j},b)\;|\;c_{i},a,\mF^{k}\right]\right|\\
		\leq& 2R_{\max}+2\gamma V_{\max}\\
		= & 2V_{\max}.
	\end{align*}
\end{proof}

\subsubsection{A contraction operator}

The following Lemma states that the joint Bellman-NN operator $G$ is a contraction with
modulus $\gamma$, and has a unique fixed point that is bounded.

\begin{lem}[Contraction of the Joint-Bellman-NN operator]
	\label{lem:G-contraction}For each fixed $h>0,$ the operator $G$
	defined in Eq. (\ref{eq:Bellman-NN-operator}) is a contraction with
	modulus $\gamma$ with
	the supremum norm. There exists a unique function $q_{h}^{*}$
	such that 
	\[
	(Gq_{h}^{*})(c_{i},a)=q_{h}^{*}(c_{i},a),\qquad\forall(c_{i},a)\in\mZ_{h},
	\]
	where $ \| q_{h}^* \|_\infty \le V_{\max} $.
\end{lem}
\begin{proof}
	Let $\mathcal{D}$ be the set of all functions $q:\mX_{h}\times\mA\to\real$
	such that $\left\Vert q\right\Vert _{\infty}\le V_{\max}.$ 	We first show that the operator $G$ maps $\mathcal{D}$ into itself.  Take any
	$q\in\mathcal{D}$, and fix an arbitrary $a\in\mA$. For any $c_{i}\in\mX_{h}$,
	we have 
	\begin{align*}
		\left|(Gq)(c_{i},a)\right| & =\left|r(c_{i},a)+\gamma\E\left[\max_{b\in\mA}(\Gamma_{\text{NN}}q)(x',b)|c_{i},a\right]\right|\\
		& \leq\left|r(c_{i},a)\right|+\gamma\left|\int_{\mX}\left[\max_{b\in\mA}\left(\sum_{j=1}^{N_{h}}K(y,c_{j})q(c_{j},b)\right)\right]p(y|c_{i}a)\lambda(dy)\right|\\
		& \le\left|r(c_{i},a)\right|+\gamma\int_{\mX}\left[\max_{b\in\mA}\sum_{j=1}^{N_{h}}K(y,c_{j})\left|q(c_{j},b)\right|\right]p(y|c_{i}a)\lambda(dy)\\
		& \le R_{\max}+\gamma V_{\max}\\
		& =V_{\max},
	\end{align*}
	where the last step follows from the definition of $V_{\max}$. This
	means that $Gq\in\mathcal{D}$, so $ G $ maps $ \mathcal{D} $ to itself.
	
	Now, by the definition of $G$ in Eq.~(\ref{eq:Bellman-NN-operator}), 	
	$\forall q,q'\in \mathcal{D},$ we have 
	\begin{align*}
		\left\Vert Gq-Gq'\right\Vert _{\infty} & =\max_{i\in[n],a\in\mA}\left|(Gq)(c_{i},a)-(Gq')(c_{i},a)\right|\\
		& \leq \gamma\max_{i\in[n],a\in\mA}\left|\E\left[\max_{b\in\mA}\left(\sum_{j=1}^{n}K(x',c_{j})\left(q(c_{j},b)-q'(c_{j},b)\right)\right)\;|\;c_{i},a\right]\right|\\
		& \leq\gamma\max_{i\in[n],a\in\mA}\E\left[\max_{b\in\mA}\left(\sum_{j=1}^{n}K(x',c_{j})\left|q(c_{j},b)-q'(c_{j},b)\right|\right)\;|\;c_{i},a\right]\\
		& \leq\gamma\max_{i\in[n],a\in\mA}\E\left[\max_{b\in\mA}\left(\sum_{j=1}^{n}K(x',c_{j})\left\Vert q-q'\right\Vert _{\infty}\right)\;|\;c_{i},a\right]\\
		& \leq\gamma\left\Vert q-q'\right\Vert _{\infty}
	\end{align*}
	Therefore $G$ is indeed a contraction on $\mathcal{D}$ with respect to the supremum norm. The Banach fixed point theorem guarantees that $G$ has a unique fixed point
	$q_{h}^{*}\in\mathcal{D}$. This completes the proof.
\end{proof}

\subsubsection{Discretization error}

For each $q\in C(\mZ_{h})$, we can obtain an extension to the original
continuous state space via the Nearest Neighbor operator. That is,
define 
\[
Q(x,a)=(\Gamma_{\text{NN}}q)(x,a),\forall(x,a)\in\mZ.
\]
The following lemma characterizes the distance between the optimal
action-value function $Q^{*}$ and the extension of the fixed-point
of the joint NN-Bellman operator $G$ to the space $\mZ.$
\begin{lem}[Discretization error]
	\label{lem:error_discrete} Define
	\[
	Q_{h}^{*}=\Gamma_{\text{NN}}q_{h}^{*}.
	\]
	Let $Q^{*}$ be the optimal action-value function for the original
	MDP. Then we have 
	\[
	\left\Vert Q_{h}^{*}-Q^{*}\right\Vert \leq \beta Ch,
	\]
	where $C=M_r+\gamma V_{\max}M_p$ and $\beta=\frac{1}{1-\gamma}$.
\end{lem}
\begin{proof}
	Consider an operator $H$ on $C(\mZ)$ defined as follows:
	\begin{align}
		(HQ)(x,a) & =(\Gamma_{\text{NN}}(FQ))(x,a)\nonumber \\
		& =\sum_{i=1}^{n}K(x,c_{i})\left\{ r(c_{i},a)+\gamma\E\left[\max_{b\in\mA}Q(x',b)\;|\;c_{i},a\right]\right\} \label{eq:NN-Bellman-Operator}
	\end{align}
	We can show that $H$ is a contraction operator with modulus $\gamma$.
	\begin{align*}
		\left\Vert HQ_{1}-HQ_{2}\right\Vert _{\infty} & =\max_{a\in\mA}\sup_{x\in\mX}\left|(HQ_{1})(x,a)-(HQ_{2})(x,a)\right|\\
		& =\gamma\max_{a\in\mA}\sup_{x\in\mX}\left|\E\left[\max_{b\in\mA}\left(\sum_{i=1}^{n}K(x,c_{i})\left(Q_{1}(x',b)-Q_{2}(x',b)\right)\right)\;|\;c_{i},a\right]\right|\\
		& \leq\gamma\max_{a\in\mA}\sup_{x\in\mX}\E\left[\max_{b\in\mA}\left(\sum_{i=1}^{n}K(x,c_{i})\left|Q_{1}(x',b)-Q_{2}(x',b)\right|\right)\;|\;c_{i},a\right]\\
		& \leq\gamma\max_{a\in\mA}\sup_{x\in\mX}\E\left[\max_{b\in\mA}\left(\sum_{i=1}^{n}K(x,c_{i})\left\Vert Q_{1}-Q_{2}\right\Vert _{\infty}\right)\;|\;c_{i},a\right]\\
		& =\gamma\left\Vert Q_{1}-Q_{2}\right\Vert _{\infty}
	\end{align*}
	We can conclude that $H$ is a contraction operator mapping $C(\mZ)$
	to $C(\mZ)$. Thus $H$ has a unique fixed point $\tilde{Q}\in C(\mZ).$
	Note that 
	\[
	H(\Gamma_{\text{NN}}q)=\Gamma_{\text{NN}}(F(\Gamma_{\text{NN}}q))=\Gamma_{\text{NN}}(Gq),
	\]
	we thus have 
	\[
	HQ_{h}^{*}=H\left(\Gamma_{\text{NN}}q_{h}^{*}\right)=\Gamma_{\text{NN}}(Gq_{h}^{*})=\Gamma_{\text{NN}}(q_{h}^{*})=Q_{h}^{*}.
	\]
	That is, the fixed point of $H$ is exactly the extension of the fixed
	point of $G$ to $\mZ.$ Therefore, we have 
	\begin{align*}
		\left\Vert Q_{h}^{*}-Q^{*}\right\Vert _{\infty} & =\left\Vert HQ_{h}^{*}-HQ^{*}+HQ^{*}-Q^{*}\right\Vert _{\infty}\\
		& \leq\left\Vert HQ_{h}^{*}-HQ^{*}\right\Vert _{\infty}+\left\Vert HQ^{*}-Q^{*}\right\Vert _{\infty}\\
		& \leq\gamma\left\Vert Q_{h}^{*}-Q^{*}\right\Vert _{\infty}+\left\Vert HQ^{*}-Q^{*}\right\Vert _{\infty}.
	\end{align*}
	It follows that 
	\begin{align*}
		\left\Vert Q_{h}^{*}-Q^{*}\right\Vert _{\infty} & \leq\frac{1}{1-\gamma}\left\Vert HQ^{*}-Q^{*}\right\Vert _{\infty}\\
		& =\frac{1}{1-\gamma}\left\Vert \Gamma_{\text{NN}}(FQ^{*})-Q^{*}\right\Vert _{\infty}\\
		& =\frac{1}{1-\gamma}\left\Vert \Gamma_{\text{NN}}(Q^{*})-Q^{*}\right\Vert _{\infty}\\
		& =\frac{1}{1-\gamma}\sup_{x\in\mX}\max_{a\in\mA}\left|\sum_{i=1}^{n}K(x,c_{i})Q^{*}(c_{i},a)-Q^{*}(x,a)\right|
	\end{align*}
	Recall that $Q^{*}(\cdot,a)$ is Lipschitz with parameter $C=M_r+\gamma V_{\max}M_p$ (see Lemma \ref{lem:Q_Lip}), i.e., for each $a\in\mA$,
	\[
	\left|Q^{*}(x,a)-Q^{*}(y,a)\right|\leq C\rho(x,y).
	\]
	From the state space discretization step, we know that the finite
	grid $\{c_{i}\}_{i=1}^{N_{h}}$ is an $h$-net in $\mX.$
	Therefore, for each $x\in\mX,$ there exists $c_{i}\in\mX_{h}$ such
	that 
	\[
	\rho(x,c_{i})\leq h.
	\]
	Thus $\sum_{i=1}^{n}K(x,c_{i})=1.$ Recall our assumption that the weighting function satisfies $K(x,y)=0$ if $\rho(x,y)\geq h$. For each $a\in\mA$, then we have
	\begin{align*}
		\sup_{x\in\mX}\left|\sum_{i=1}^{n}K(x,c_{i})Q^{*}(c_{i},a)-Q^{*}(x,a)\right| & =\sup_{x\in\mX}\left|\sum_{c_{i}\in\ball_{x,h}}K(x,c_{i})Q^{*}(c_{i},a)-Q^{*}(x,a)\right|\\
		& \leq\sup_{x\in\mX}\sum_{c_{i}\in\ball_{x,h}}K(x,c_{i})\left|Q^{*}(c_{i},a)-Q^{*}(x,a)\right|\\
		& \leq Ch
	\end{align*}
	This completes the proof. 
\end{proof}

\subsection{Applying the Stochastic Approximation Theorem to NNQL} \label{subsec:SA_to_NNQL}

We first apply Theorem \ref{thm:PAC_linear_inf} to establish
that NNQL converges to a neighborhood of $q_{h}^* $, the fixed point of the Joint Bellman-NN operator $ G $, after a sufficiently large number of iterations. This is summarized in the following
theorem.
\begin{thm}
	\label{thm:conve-rate-discrete}Let Assumptions \ref{assu:MDP-Reularity}
	and \ref{assu:cover-time} hold. 
	Then for each $0<\varepsilon<{2 V_{\max} \beta}$, after 
	\[
	k=\frac{192V_{\max}^{3}\beta^{4}}{\varepsilon^{3}}\log\left(\frac{128dV^2_{\max}\beta^{4}}{\delta\varepsilon^{2}}\right)+\frac{4V_{\max}(\beta-1)}{\varepsilon}
	\]
	iterations of Nearest-Neighbor Q-learning, with probability at least
	$1-\delta$, we have 
	\[
	\left\Vert q^{k}-q_{h}^{*}\right\Vert _{\infty}\leq\beta(\delta_{1}+\delta_{2}V_{\max})+\varepsilon.
	\]
\end{thm}
\begin{proof}
	We will show that NNQL satisfies the assumptions of Theorem \ref{thm:PAC_linear_inf}.
	It follows from Lemma~\ref{lem:G-contraction} that the operator $G$
	is a $\gamma$-contraction with a unique fixed point $\|q^*_h\|_\infty\leq V_{\max}$. For each $Y_{t}\in\mathcal{Y}_{k}(c_{i},a)$,
	we have $\rho(Y_{t},c_{i})\leq h,\;a_{t}=a$. Thus
	\begin{align*}
		&\left|\E\left[w^{k+1}(c_{i},a) | \mF^{k}\right]\right| \\
		=&\bigg|\E\Big[\frac{1}{|\mathcal{Y}_{k}(c_{i},a)|}\sum_{Y_{t}\in\mathcal{Y}_{k}(c_{i},a)}\Big[R_t+\gamma\max_{b\in\mA}(\Gamma_{\text{NN}}q^{k})(Y_{t+1},b)\Big]\;|\;\mF^{k}\Big]\\
		& \qquad -r(c_{i},a)  -\gamma\E\Big[\max_{b\in\mA}(\Gamma_{\text{NN}}q^{k})(x',b)\;|\;c_{i},a,\mF^{k}\Big] \bigg|\\
		\leq & \bigg| \E\Big[\frac{1}{|\mathcal{Y}_{k}(c_{i},a)|}\sum_{Y_{t}\in\mathcal{Y}_{k}(c_{i},a)} R_t -r(c_{i},a)\;|\;\mF^{k}\Big]\bigg|\\
		& +\gamma\bigg|\E\Big[\frac{1}{|\mathcal{Y}_{k}(c_{i},a)|}\sum_{Y_{t}\in\mathcal{Y}_{k}(c_{i},a)}\max_{b\in\mA}(\Gamma_{\text{NN}}q^{k})(Y_{t+1},b)\;|\;\mF^{k}\Big] - \E\Big[\max_{b\in\mA}(\Gamma_{\text{NN}}q^{k})(x',b)\;|\;c_{i},a,\mF^{k}\Big]\bigg|.
	\end{align*}
	We can bound the first term on the RHS by using Lipschitz continuity
	of the reward function: 
	\begin{align*}
		& \quad\bigg|\E\Big[\frac{1}{|\mathcal{Y}_{k}(c_{i},a)|}\sum_{Y_{t}\in\mathcal{Y}_{k}(c_{i},a)}R_t-r(c_{i},a)\;\big|\;\mF^{k}\Big]\bigg|\\
		& =\bigg| \E\bigg[ \E\Big[\frac{1}{|\mathcal{Y}_{k}(c_{i},a)|}\sum_{Y_{t}\in\mathcal{Y}_{k}(c_{i},a)}\big(R_t-r(c_{i},a)\big) \;\big|\; \mathcal{Y}_{k}, \mF^{k} \Big] \;\Big|\; \mF^{k} \bigg] \bigg|\\
		& =\bigg| \E\bigg[ \frac{1}{|\mathcal{Y}_{k}(c_{i},a)|} \sum_{Y_{t}\in\mathcal{Y}_{k}(c_{i},a)}\big(r(Y_t,a)-r(c_{i},a)\big)  \;\Big|\; \mF^{k} \bigg] \bigg|\\
		& \leq\E\Big[\frac{1}{|\mathcal{Y}_{k}(c_{i},a)|}\sum_{Y_{t}\in\mathcal{Y}_{k}(c_{i},a)}\big|r(Y_{t},a)-r(c_{i},a)\big| \;\Big|\; \mF^{k} \Big]\\
		& \leq\E\Big[\frac{1}{|\mathcal{Y}_{k}(c_{i},a)|}\sum_{Y_{t}\in\mathcal{Y}_{k}(c_{i},a)}M_{r}\rho(Y_{t},c_{i}) \;\Big|\; \mF^{k} \Big] & \text{Lipschitz continuity of $ r(\cdot, a) $ }\\
		& \leq M_{r}h & \rho(Y_{t},c_{i}) \leq h
	\end{align*}
	The second term on the right hand side can be bounded as follows:
	\begin{align*}
		& \left|\E\left[\frac{1}{|\mathcal{Y}_{k}(c_{i},a)|}\sum_{Y_{t}\in\mathcal{Y}_{k}(c_{i},a)}\max_{b\in\mA}(\Gamma_{\text{NN}}q^{k})(Y_{t+1},b)\;|\;\mF^{k}\right]-\E\left[\max_{b\in\mA}(\Gamma_{\text{NN}}q^{k})(x',b)\;|\;c_{i},a,\mF^{k}\right]\right|\\
		\le & \E\Bigg[\Bigg|\frac{1}{|\mathcal{Y}_{k}(c_{i},a)|}\sum_{Y_{t}\in\mathcal{Y}_{k}(c_{i},a)}\int_{\mX}\left[\max_{b\in\mA}(\Gamma_{\text{NN}}q^{k})(y,b)\right]p(y\;|\;Y_{t},a)\lambda(dy) \\
		& \qquad  -\int_{\mX}\left[\max_{b\in\mA}(\Gamma_{\text{NN}}q^{k})(y,b)\right]p(y\;|\;c_{i},a)\lambda(dy)\Bigg| \;|\; \mF^{k}\Bigg]\\
		= & \E\left[\left|\frac{1}{|\mathcal{Y}_{k}(c_{i},a)|}\sum_{Y_{t}\in\mathcal{Y}_{k}(c_{i},a)}\int_{\mX}\left[\max_{b\in\mA}(\Gamma_{\text{NN}}q^{k})(y,b)\right]\left(p(y\;|\;Y_{t},a)-p(y\;|\;c_{i},a)\right)\lambda(dy)\right|\;|\;\mF^{k}\right]\\
		\leq & \E\left[\frac{1}{|\mathcal{Y}_{k}(c_{i},a)|}\sum_{Y_{t}\in\mathcal{Y}_{k}(c_{i},a)}\int_{\mX}\left[\max_{b\in\mA}(\Gamma_{\text{NN}}q^{k})(y,b)\right]\left|p(y\;|\;Y_{t},a)-p(y\;|\;c_{i},a)\right|\lambda(dy)\;|\;\mF^{k}\right]\\
		\leq & \E\left[\frac{\sup_{y\in\mX}\max_{b\in\mA}(\Gamma_{\text{NN}}q^{k})(y,b)}{|\mathcal{Y}_{k}(c_{i},a)|}\sum_{Y_{t}\in\mathcal{Y}_{k}(c_{i},a)}\int_{\mX}\left|p(y\;|\;Y_{t},a)-p(y\;|\;c_{i},a)\right|\lambda(dy)\;|\;\mF^{k}\right]\\
		= & \sup_{y\in\mX}\max_{b\in\mA}\left|\sum_{j=1}^{n}K(y,c_{j})q^{k}(c_{j},b)\right|\times\E\left[\frac{1}{|\mathcal{Y}_{k}(c_{i},a)|}\sum_{Y_{t}\in\mathcal{Y}_{k}(c_{i},a)}\int_{\mX}\left|p(y\;|\;Y_{t},a)-p(y\;|\;c_{i},a)\right|\lambda(dy)\;|\;\mF^{k}\right]\\
		\leq & \max_{c_{j}\in\mX_{h}}\max_{b\in\mA}\left|q^{k}(c_{j},b)\right|\times\E\left[\frac{1}{|\mathcal{Y}_{k}(c_{i},a)|}\sum_{Y_{t}\in\mathcal{Y}_{k}(c_{i},a)}\int_{\mX}W_{p}(y)\rho(Y_{t},c_{i}) \lambda(dy)\;|\;\mF^{k}\right]  \\
		\leq & \left\Vert q^{k}\right\Vert _{\infty}\times\E\left[\frac{1}{|\mathcal{Y}_{k}(c_{i},a)|}\sum_{Y_{t}\in\mathcal{Y}_{k}(c_{i},a)}\int_{\mX}W_{p}(y)h\lambda(dy)\;|\;\mF^{k}\right] \\
		\leq & \left\Vert q^{k}\right\Vert _{\infty}hM_{p}.
	\end{align*}
	Putting together, we have 
	\[
	\left|\E\left[w^{k+1}(c_{i},a)\;|\;\mF^{k}\right]\right|\leq h(M_{r}+\gamma M_{p}\left\Vert q^{k}\right\Vert _{\infty}),\quad\forall(c_{i},a)\in\mZ_{h}.
	\]
	Hence the noise $w^{k+1}$ satisfies Assumption \ref{a:1} of Theorem \ref{thm:PAC_linear_inf}
	with 
	\[
	\delta_{1}=hM_{r},\delta_{2}=h\gamma M_{p}.
	\]

	From Lemma \ref{lem:stability}, we have 
	\begin{align*}
		\left|w^{k+1}(c_{i},a)-\E\left[w^{k+1}(c_{i},a)\;|\;\mF^{k}\right]\right|& \leq2V_{\max}, \quad\forall(c_{i},a)\in\mZ_{h}, \\
		\left\Vert q^{k}\right\Vert _{\infty} & \leq V_{\max}. 
	\end{align*}
	Therefore, the remaining Assumptions \ref{a:3}--\ref{a:4} of Theorem \ref{thm:PAC_linear_inf} are satisfied. And the update algorithm uses the learning rate suggested in Theorem \ref{thm:PAC_linear_inf}. Therefore, we conclude that for each  $0<\varepsilon<2V_{\max} \beta$
	(since $\beta \geq 1$ and hence $2 V_{\max} \beta \leq \min\{2 V_{\max} \beta, 4V_{\max} \beta^2\}$), 
	after 
	\[
	k=\frac{192V_{\max}^{3}\beta^{4}}{\varepsilon^{3}}\log\left(\frac{128N_{h}\left|\mA\right|V_{\max}\beta^{4}}{\delta\varepsilon^{2}}\right)+\frac{6V_{\max}(\beta-1)}{\varepsilon}
	\]
	iterations of (\ref{eq:NNQL}), with probability at least $1-\delta,$ we have
	\[
	\left\Vert q^{k}-q_{h}^{*}\right\Vert _{\infty}\leq\beta h(M_{r}+\gamma M_{p}V_{\max})+\varepsilon.
	\]
\end{proof}

To prove Theorem \ref{thm:NN Q-learning}, we need the following result which bounds the number of time steps required to visit all ball-actions $k$ times with high probability.

\begin{lem} \label{lem:cover-time-asy} (Lemma $14$ in~\cite{Azar2011SQLlong}, rephrased)
	Let Assumption \ref{assu:cover-time} hold. Then for all initial state $x_0\in \mX,$ and for each integer $k\geq 4,$ after a run of 
	$T=8 k L_h\log\frac{1}{\delta}$ steps, the finite space $\mZ_h$ is covered at least $k$ times under the policy $\pi$ with probability at least 
	$1-\delta$ for any $\delta \in (0, \frac1e)$.
\end{lem}

Now we are ready to prove Theorem \ref{thm:NN Q-learning}.
\begin{proof}
	We denote by $\tilde{Q}_{h}^{k}$ the extension of $q^{k}$ to $\mZ$ via
	the nearest neighbor operation, i.e., $\tilde{Q}_{h}^{k}=\Gamma_{\text{NN}}q^{k}.$ Recall that $Q_{h}^{*}$ is the extension of $q^*_h$ (the fixed point of $Gq=q$) to $\mZ.$
	We have 
	\begin{align*}
		\left\Vert \tilde{Q}_{h}^{k}-Q^{*}\right\Vert _{\infty} & \leq\left\Vert \tilde{Q}_{h}^{k}-Q_{h}^{*}\right\Vert _{\infty}+\left\Vert Q_{h}^{*}-Q^{*}\right\Vert _{\infty}\\
		& =\left\Vert \Gamma_{\text{NN}}q^{k}-\Gamma_{\text{NN}}q_{h}^{*}\right\Vert _{\infty}+\left\Vert Q_{h}^{*}-Q^{*}\right\Vert _{\infty} &  \\
		& \leq\left\Vert q^{k}-q_{h}^{*}\right\Vert _{\infty}+\left\Vert Q_{h}^{*}-Q^{*}\right\Vert _{\infty} &  & \Gamma_{\text{NN}}\text{ is non-expansive}\\
		& \leq\left\Vert q^{k}-q_{h}^{*}\right\Vert _{\infty}+\beta Ch &  & \text{Lemma \ref{lem:error_discrete}}
	\end{align*}
	It follows from Theorem \ref{thm:conve-rate-discrete} that, after
	\[
	k=\frac{192V_{\max}^{3}\beta^{4}}{\varepsilon_0^{3}}\log\left(\frac{128N_{h}\left|\mA\right|V^2_{\max}\beta^{4}}{\delta\varepsilon_0^{2}}\right)+\frac{6V_{\max}(\beta-1)}{\varepsilon_0}
	\]
	iterations, with probability at least $1-\delta,$ we have 
	\[
	\left\Vert \tilde{Q}_{h}^{k}-Q^{*}\right\Vert _{\infty}\leq\beta h(M_{r}+\gamma M_{p}V_{\max})+\beta Ch+\varepsilon_0=2\beta Ch+\varepsilon_0.
	\]
	By setting $\varepsilon_0=\frac{\varepsilon}{2}$ and $h^*(\varepsilon)=\frac{\varepsilon}{4\beta C},$ we have $\left\Vert \tilde{Q}_{h}^{k}-Q^{*}\right\Vert _{\infty}\leq\varepsilon.$ Let $N_{h^{*}}$ be the $h^*$-covering number of the
	metric space $(\mX,\rho).$ Plugging the result of Lemma \ref{lem:cover-time-asy} concludes the proof of Theorem~\ref{thm:NN Q-learning}.
	
\end{proof}

\section{Proof of Corollary~\ref{cor:random_policy}} \label{sec:proof_random_policy}

\begin{proof}
	Since the probability measure $ \nu $ is uniform over $ \mX,$ we have $ \nu_{\min} \triangleq \min_{i\in[N_{h^*}]} \nu(\ball_i)=O(\frac{1}{N_{h^*}}). $ By Proposition~\ref{prop:cover_time}, the expected covering time of a purely random policy is upper bounded by 
	$$ L_{h^*}=O\bigg(\frac{mN_{h^*}|\mA|}{\psi}\log(N_{h^*}|\mA|)\bigg).$$
    
	By Proposition~4.2.12 in \cite{vershynin_hdpbook}, the covering number $ N_{h^*} $ of $ \mX=[0,1]^d $ scales as $ O\big((1/h^*)^d\big), $ which is $ O\big((\beta/\varepsilon)^d\big)$ with $ h^*=\frac{\varepsilon}{4\beta C}.$
	
	From Theorem~\ref{thm:NN Q-learning}, with probability at least $1-\delta$ we have $\left\Vert Q_{h^{*}}^{T}-Q^{*}\right\Vert _{\infty}\leq\varepsilon$, after at most \[
	T=O\Bigg(\frac{\big|\mA\big|\beta^{d+7}}{\varepsilon^{d+3}}\log\left(\frac{2}{\delta}\right)\log\bigg(\frac{\big|\mA\big|\beta^{d}}{\varepsilon^{d}} \bigg)\log\left(\frac{\big|\mA\big|\beta^{d+6}}{\delta\varepsilon^{d+2}}\right)
	\Bigg) 
	\]
	steps. Corollary~\ref{cor:random_policy} follows after absorbing the dependence on $ |\mA|, d, \beta $ into $ \kappa \equiv \kappa(|\mA|, d, \beta) $ and doing some algebra.
\end{proof}

\section{Proof of Theorem~\ref{thm:lower_bound}} \label{sec:proof_lower_bound}

We prove Theorem~\ref{thm:lower_bound} by connecting the problem of estimating the value function in MDPs to the problem of non-parametric regression, and then leveraging known minimax lower bound for the latter. In particular, we show that a class of non-parametric regression problem can be embedded in an MDP problem, so any algorithm for the latter can be used to solve the former. Prior work on non-parametric regression\cite{tsybakov2009nonparm,stone1982optimal} establishes that a certain number of observations is \emph{necessary} to achieve a given accuracy using \emph{any} algorithms, hence leading to a corresponding necessary condition for the sample size of estimating the value function in an MDP problem.

We now provide the details. 

\noindent{\bf Step 1. Non-parametric regression}

Consider the following non-parametric regression problem:
Let $ \mX:=[0,1]^{d} $ and assume that we have $T$ independent pairs of random variables $(x_{1},y_{1}),\ldots,(x_{T},y_{T})$
such that 
\begin{equation}
\E\left[y_{t}|x_{t}\right]=f(x_{t}),\qquad x_{t}\in\mX \label{eq:regression}
\end{equation}
where $x_{t}\sim\text{uniform}(\mX)$ and $f:\mX\to\real$ is the
unknown regression function. Suppose that the conditional distribution of $ y_t $ given $ x_t=x $ is a Bernoulli distribution with mean $ f(x) $. We also assume that $ f $ is $1 $-Lipschitz continuous with respect to the Euclidean norm, i.e., 
\[
 |f(x)-f(x_0)|\leq  \vert x-x_0 \vert, \quad \forall x,x_0\in \mX.
 \]
Let $ \mathcal{F} $ be the collection of all $ 1 $-Lipschitz continuous function on $ \mathcal{X}$, i.e., 
\[
\mF=\text{Lip}\left(\mX,1 \right) = \left\{ h|\text{\ensuremath{h} is a 1-Lipschitz function on \ensuremath{\mX}}\right\},
\] 
where $ \text{Lip} (\cdot, \cdot) $ is as defined in Section~\ref{sec:prelim}. The goal is to estimate~$f$ given the observations $(x_{1},y_{1}),\ldots,(x_{T},y_{T})$ and the prior knowledge that $ f\in \mF $. 

It is easy to verify that the above problem is a special case of the non-parametric regression problem considered in the work by Stone~\cite{stone1982optimal} (in particular, Example~2 therein).
Let $ \hat{f}_T $ denote an arbitrary (measurable) estimator of $ f $ based on the training samples $(x_{1},y_{1}),\ldots,(x_{T},y_{T})$.
By Theorem~1 in~\cite{stone1982optimal}, we have the following result: there exists a $ c>0 $ such that 
\begin{align}
\lim_{T\to\infty}\inf_{\hat{f}_{T}}\sup_{f\in\mF}\P\bigg(\big\Vert \hat{f}_{T}-f\big\Vert _{\infty}\ge c\Big(\frac{\log T}{T}\Big)^{\frac{1}{2+d}}\bigg)=1,
\end{align}
where infimum is over all possible estimators $ \hat{f}_T $. 

Translating this result to the non-asymptotic regime, we obtain the
following theorem.
\begin{thm}
	\label{thm:regression_lower_bound}Under the above assumptions,
	for any number $\delta\in(0,1)$, there exits some numbers $ c>0 $ and $T_{\delta}$
	such that 
	\[
	\inf_{\hat{f}_{n}}\sup_{f\in\mF}\P\bigg(\big\Vert \hat{f}_{T}-f\big\Vert _{\infty}\ge c\Big(\frac{\log T}{T}\Big)^{\frac{1}{2+d}}\bigg) \ge \delta, \qquad\text{for all \ensuremath{T\ge T_{\delta}}}.
	\]
\end{thm}

\medskip
\noindent{\bf Step 2. MDP}

Consider a class of (degenerate) discrete-time discounted MDPs $\left(\mX,\mA,p,r,\gamma\right)$
where
\begin{align*}
\mX & =[0,1]^{d},\\
\mA & \text{ is finite},\\
p(\cdot|x,a) & =p(\cdot|x)\text{ is uniform on \ensuremath{\mX} for all \ensuremath{x,a}},\\
r(x,a) & =r(x)\text{ for all \ensuremath{a}},\\
\gamma & \in(0,1).
\end{align*}
In words, the transition is uniformly at random and independent
of the current state and the actions taken, and the expected reward
is independent on the action taken but dependent on the current state. 

Let $R_{t}$ be the observed reward at step $t$. We assume that the distribution of $R_{t}$ given $x_{t}$ is $\text{Bernoulli}\big(r(x_{t})\big)$, independently of $(x_{1},x_{2},\ldots,x_{t-1})$.
The expected reward function $
r(x_{t})=\E\left[R(x_{t})|x_{t}\right]
$
is assumed to be $1$-Lipschitz and bounded.

It is easy to see that for all $x\in\mX$, $a\in\mA$, 
\begin{align}
Q^{*}(x,a)=V^{*}(x) & =r(x)+\gamma\E\left[V^{*}(X')|x\right] \notag\\
& =r(x)+\gamma\int_{\mX}V^{*}(y)p(y|x)dy \notag\\
& =r(x)+\gamma\underbrace{\int_{\mX}V^{*}(y)dy}_{C}, \label{eq:QVr}
\end{align}
where the last step holds because $ p(\cdot|x) $ is uniform.
Integrating both sides over $\mX$, we obtain
\begin{align*}
C & =\int_{\mathcal{X}}r(x)dx+\gamma C,
\end{align*}
whence
\begin{align*}
C & =\frac{1}{1-\gamma}\int_{\mathcal{X}}r(x)dx.
\end{align*}
It follows from equation~\eqref{eq:QVr} that 
\begin{equation} \label{eq:Vr}
V^{*}(x)=r(x)+\frac{\gamma}{1-\gamma}\int_{\mathcal{X}}r(y)dy,\qquad\forall x\in\mX,
\end{equation}
and
\begin{equation} \label{eq:rV}
r(x)=V^{*}(x)-\gamma\int_{\mX}V^{*}(y)dy,\qquad\forall x\in\mX.
\end{equation}

Regardless of the exploration policy used, the sample trajectory $(x_{1},x_{2},\ldots,x_{T})$
is i.i.d.~and uniformly random over $\mX$, and the observed rewards
$(R_{1},R_{2},\ldots,R_{T})$ are independent.

\smallskip
\noindent{\bf Step 3. Reduction from regression to MDP}

Given a non-parametric regression problem as described in Step $1$,
we may reduce it to the problem of estimating the value function $V^{*}$
of the MDP described in Step $2$. To do this, we set
\begin{align*}
r(x) & =f(x)-\gamma\int_{\mX}f(y)dy, \qquad\forall x\in\mX
\end{align*}
and
\begin{align*}
R_{t} & =y_{t},\qquad t=1,2,\ldots,T.
\end{align*}
In this case, it follows from equations~\eqref{eq:Vr} and~\eqref{eq:rV} that the value function is given by $V^{*}=f$. Moreover, the expected reward function $r(\cdot)$ is $1$-Lipschitz as
it is just $f(\cdot)$ minus a constant, so the assumptions of the MDP in Step $ 2 $ are
satisfied. This reduction shows that the MDP problem is at least as hard as
the nonparametric regression problem, so a lower bound for the latter
is also a lower bound for the former. Applying Theorem~\ref{thm:regression_lower_bound} yields the result
stated in Theorem~\ref{thm:lower_bound}.

\end{document}